\pdfoutput=1
\documentclass[10pt]{article}
\usepackage{amsmath,amsfonts,amsthm,mathrsfs,amssymb}
\usepackage{enumerate}
\usepackage{appendix}
\usepackage{graphicx}
\usepackage{subfigure}
\usepackage{caption}
\usepackage{textcomp}
\usepackage{anysize}

\usepackage{booktabs}

\usepackage{adjustbox/adjustbox}
\usepackage{multirow}

\usepackage{array}

\newcolumntype{M}[1]{>{\centering\arraybackslash}m{#1}}
\newcolumntype{N}{@{}m{0pt}@{}}

\newcommand{\specialcell}[2][c]{%
  \begin{tabular}[#1]{@{}c@{}}#2\end{tabular}}

\newtheorem{thm}{Theorem}
\newtheorem{lemma}{Lemma}
\newtheorem{pro}{Proposition}
\newtheorem{corollary}{Corollary}

\newtheorem{rem}{Remark}

\newtheorem{alg}{Algorithm}
\newtheorem{as}{Assumption}

\newcommand{\be}{\begin{equation}}
\newcommand{\ee}{\end{equation}}
\newcommand{\bs}{\begin{split}}
\newcommand{\es}{\end{split}}
\newcommand{\bea}{\begin{eqnarray*}}
\newcommand{\eea}{\end{eqnarray*}}
\newcommand{\mR}{\mathbb{R}}
\newcommand{\mN}{\mathbb{N}}
\newcommand{\mE}{\mathbb{E}}
\newcommand{\mcH}{\mathcal{H}}
\newcommand{\mcF}{\mathcal{F}}

\newcommand{\mcE}{\mathcal{E}}

\newcommand{\la}{\langle}
\newcommand{\ra}{\rangle}
\newcommand{\eref}[1] {(\ref{#1})}

\providecommand{\scal}[2]{\langle{#1},{#2}\rangle}

\begin{document}

\title{Generalization Properties and Implicit Regularization  for Multiple Passes SGM}

\author{Junhong Lin$^{*}\quad$ Raffaello Camoriano$^{\dagger*\ddagger}\quad$  Lorenzo Rosasco$^{*\ddagger}$\\
{\small{\em jhlin5@hotmail.com}$~\qquad${\em raffaello.camoriano@iit.it}$~\qquad$ {\em lrosasco@mit.edu}}$~\quad$\\[5mm]
{\footnotesize ${}^*$LCSL, Massachusetts Institute of Technology and Istituto Italiano di Tecnologia, Cambridge, MA 02139, USA}\\
{\footnotesize  ${}^\ddagger$DIBRIS, Universit\`a degli Studi di Genova, Via Dodecaneso 35, Genova, Italy}\\
{\footnotesize${}^\dagger$iCub Facility, Istituto Italiano di Tecnologia, Via Morego 30, Genova, Italy}
}
\maketitle \baselineskip 16pt

\begin{abstract}
We study the generalization properties of stochastic gradient methods for learning with convex loss functions and linearly parameterized functions. We show that, in the absence of penalizations or constraints, the stability and approximation properties of the algorithm can be controlled  by tuning either the step-size or the number of passes over the data. In this view, these parameters can be seen to  control a form of implicit regularization. Numerical results complement the theoretical findings.
\end{abstract}

\section{Introduction}

The stochastic gradient method (SGM), often called stochastic gradient descent,
has become an algorithm of choice in machine learning, because of  its simplicity and small computational cost
especially when dealing with big data sets \cite{bousquet2008tradeoffs}.

Despite   its widespread use, the  generalization properties of the variants of SGM used in practice
are relatively little understood.  Most previous works consider generalization properties of SGM with only one pass over the data, see e.g.  \cite{nemirovski2009robust} or \cite{orab14} and references therein, while in practice  multiple passes are usually considered. The effect of multiple passes has  been studied extensively for the optimization  of an empirical objective \cite{boyd2007stochastic}, but the role for generalization is less clear.
In practice, early-stopping of the number of  iterations, for example monitoring a hold-out set error,  is a
strategy often used to regularize. Moreover, the  step-size is  typically  tuned to obtain the best results. The study in this paper is a step towards grounding theoretically these commonly used  heuristics.

Our starting points are a few recent works considering the generalization properties of different variants of SGM.
One first series of results  focus on least squares, either with one \cite{ying2008online,tarres2014online,dieuleveut2014non}, or  multiple (deterministic) passes over the data \cite{rosasco2015learning}.
In the former case it is shown that,  in general,
 if  only one pass over the data is considered, then
the step-size needs to be tuned to ensure optimal results. In \cite{rosasco2015learning} it is shown that a universal step-size choice can be taken, if multiple passes are considered. In this case, it is the stopping time that needs to be tuned.

In this paper, we are interested in general, possibly non smooth,  convex loss functions.
The analysis for least squares heavily exploits properties of the loss and does not generalize to this
broader setting. Here, our starting points are the results in \cite{lin2015iterative,hardt2015train, orab14}  considering convex loss functions.
In \cite{lin2015iterative},  early stopping of a (kernelized) batch subgradient method is analyzed,  whereas in \cite{hardt2015train}
the stability properties of SGM for smooth loss functions are considered in a general stochastic optimization setting
and certain convergence results are derived. In \cite{orab14},  a more complex variant of SGM is analyzed and shown to achieve optimal rates.

Since we are interested in analyzing regularization and generalization properties of SGM, in this paper we consider a general non-parametric setting.
In this latter setting, the effects of regularization are typically more evident since
it can  directly  affect  the convergence rates.
In this context, the difficulty of a problem is characterized by an assumption on the approximation error. Under this condition, the need for regularization becomes clear.
Indeed, in the absence of other constraints,
%
the good performance of the algorithm relies on a bias-variance trade-off that can be controlled
by suitably choosing the  step-size and/or the number of passes. These latter parameters can be seen to act as  regularization parameters.
Here, we refer to the regularization as `implicit', in the sense that it is achieved neither by
penalization nor by adding explicit constraints.
The two main variants of the algorithm  are the same as in least squares: one pass over the data with tuned step-size, or,
fixed step-size choice and number of passes appropriately tuned. While in principle  optimal parameter tuning requires
explicitly solving a bias-variance trade-off, in practice adaptive choices can be implemented by cross-validation.
In this case, both algorithm variants achieve optimal results, but different computations are entailed.
In the first case, multiple single pass SGM need to be considered with different step-sizes, whereas  in the second case, early stopping is used.
Experimental results, complementing the theoretical analysis, are given and provide further
insights on the properties of the algorithms.

The rest of the paper is organized as follows. In Section \ref{sec:setting}, we describe the supervised learning setting and the algorithm, and in Section \ref{sec:theory}, we state and discuss our main results. The proofs are postponed to the supplementary material. In Section \ref{sec:simulations}, we present some numerical experiments on real datasets.

{\bf Notation}. For notational simplicity, $[m]$ denotes $\{1,2,\cdots,m\}$ for any $m\in \mN$.
The notation $a_k\lesssim b_k$ means that there exists a universal constant $C>0$ such that $a_k \leq Cb_k$ for all $k\in \mN.$ Denote by  $\lceil a \rceil$ the smallest integer greater than $a$ for any given $a \in \mR.$

\section{Learning with SGM}\label{sec:setting}
In this section, we introduce the supervised learning problem and the SGM algorithm.

\paragraph{Learning Setting.}
Let $X$ be a probability space and $Y$ be a subset of $\mR$.
Let $\rho$ be a probability measure on $Z=X\times Y.$
Given a  measurable loss function $V: \mR\times \mR \to \mR_{+},$
 the associated
 expected
risk $\mathcal{E} = \mathcal{E}^{V}$  is defined as
$$\mathcal{E} (f)=\int_Z V(y, f(x)) d\rho.$$
The distribution $\rho$ is assumed to be fixed, but unknown, and the goal is to find a function minimizing the expected risk
given  a sample ${\bf z} =\{z_i=(x_i, y_i)\}_{i=1}^m$ of size $m\in\mN$ independently drawn according to $\rho$.
Many  classical examples of learning algorithms are based on empirical risk minimization, that is replacing
the expected risk with  the empirical risk $\mcE_{\mathbf{z}} = \mcE_{\mathbf{z}}^{V}$
defined as
$$\mcE_{\mathbf{z}} (f)=\frac{1}{m}\sum_{j=1}^{m}V(y_j, f(x_j)). $$
In this paper, we consider spaces of functions which are linearly parameterized.
Consider a possibly non-linear data representation/feature map $\Phi:X\to \mcF$, mapping the data space in $\mR^p$, $p\le \infty$, or more generally in a (real separable) Hilbert  space with inner product $\scal{\cdot}{\cdot}$ and norm $\|\cdot\|$. Then, for $w\in \mcF$ we consider functions  of the form
\be\label{eq:linfun}
f_w(x)=\scal{w}{\Phi(x)}, \quad \forall x\in X.
\ee
Examples of the above setting include the case where we consider infinite dictionaries, $\phi_j:X\to \mR$, $j=1, \dots$,  so that $\Phi(x)=(\phi_j(x))_{j = 1}^\infty$, for all $x \in X,$ $\mcF=\ell_2$
and \eqref{eq:linfun} corresponds  to $f_w=\sum_{j=1}^pw^j \phi_j$. Also, this setting includes, and indeed is equivalent to considering, functions defined by a  positive definite kernel $K: X \times X \to \mR$, in which case $\Phi(x)=K(x, \cdot) $, for all $x\in X$, $\mcF=\mcH_K$ the reproducing kernel Hilbert space associated with $K$, and \eqref{eq:linfun} corresponds to the reproducing property
\be\label{reproducingProperty}
f_w(x) = \la w, K(x,\cdot) \ra, \forall x \in X.
\ee
In the following, we assume the feature map to be measurable  and define expected and empirical risks over functions of the form \eqref{eq:linfun}. For notational simplicity,
we write $\mcE(f_w)$ as $\mcE(w)$, and $\mcE_{\bf z}(f_w)$  as $\mcE_{\bf z}(w)$.

\paragraph{Stochastic Gradient Method.}
For any fixed $y\in Y$, assume the univariate function $V(y, \cdot)$ on $\mR$ to be  convex, hence its left-hand derivative $V'_- (y, a)$ exists at every $a\in \mR$ and is non-decreasing.

\begin{alg}\label{alg:SIGD}
Given a sample ${\bf z}$,
 the stochastic gradient  method (SGM) is defined by  $w_1= 0$ and
 \be\label{SIGD}
w_{t+1}=w_t - \eta_t  V'_- (y_{j_t}, \la w_t, \Phi(x_{j_t}) \ra ) \Phi(x_{j_t}),\, t=1, \ldots, T,
\ee
 for a non-increasing sequence of step-sizes $\{\eta_t >0 \}_{t \in \mN}$ and a stopping rule $T\in \mN$.
 Here, $j_1,j_2,\cdots,j_T$ are independent and identically distributed (i.i.d.) random variables\footnote{More precisely, $j_1,j_2,\cdots,j_T$ are conditionally independent given any $\bf z$.} from the uniform distribution on $[m]$.
The (weighted) averaged iterates are defined by
$$
\overline{w}_t = { \sum_{k=1}^{t} \eta_k w_k /a_t}, \quad a_t = { \sum_{k=1}^t \eta_k}, \quad t=1, \dots, T.
$$
\end{alg}
Note that $T$ may be greater than $m$, indicating that we can use the sample more than once.
We shall write $J(t)$ to mean $\{j_1,j_2,\cdots,j_t\}$, which will be also abbreviated as $J$ when there is no confusion.

The main purpose of the paper is to  estimate the expected excess risk of the last iterate
 $$\mE_{{\bf z},J}[{\mathcal E}(w_{T}) -
 \inf_{w\in \mcF}{\mathcal E}(w)],$$
 or similarly the   expected excess risk of the averaged iterate
 $\overline{w}_T$, and study how different parameter settings in~\eqref{alg:SIGD} affect the estimates. Here, the expectation $\mE_{{\bf z},J}$ stands for taking the expectation with respect to $J$ (given any $\bf z$) first, and then the expectation with respect to $\bf z$.

\section{Implicit Regularization for SGM}\label{sec:theory}
In this section, we present and discuss our main results.
We begin in Subsection \ref{subsec:convergence} with a universal convergence  result
and then  provide  finite sample bounds for smooth loss functions in Subsection \ref{subsec:smoothBound}, and for non-smooth functions in Subsection \ref{subsec:nonsmooth}.
As  corollaries of these results we derive different implicit regularization strategies for SGM.


\subsection{Convergence}\label{subsec:convergence}
We begin presenting a convergence result, involving conditions on both  the step-sizes and the number of  iterations.  We need some basic  assumptions.

\begin{as}\label{as:Boundness}
There holds
\be\label{boundedkernel}
\kappa=\sup_{x\in X}\sqrt{\la \Phi(x),\Phi(x) \ra }<\infty.
\ee Furthermore, the loss function
 is convex with respect to its second entry, and
$|V|_0 :=\sup_{y\in Y} V(y, 0) < \infty$.
Moreover, its left-hand derivative $V'_- (y, \cdot)$ is bounded:
\be\label{boundedDeriviative} \left|V'_- (y, a)\right| \leq a_0, \qquad \forall a\in \mR, y\in Y.\ee
\end{as}
The above  conditions are common in statistical learning theory \cite{steinwart2008support,cucker2007learning}. For example,
they are satisfied for the hinge loss  $V(y,a)= |1-ya|_+=\max\{0, 1-ya\}$
  or the logistic loss  $V(y,a) = \log (1 + \mathrm{e}^{-ya})$ for all $a\in \mR$,
   if $X$ is compact and $\Phi(x)$ is continuous.

%
%
The bounded derivative condition \eref{boundedDeriviative} is implied by the requirement on the loss function to be  Lipschitz in its second entry, when $Y$ is a bounded domain. Given these  assumptions, the following result holds.

\begin{thm}
  \label{thm:convergence} If Assumption \ref{as:Boundness} holds,
   then
  $$
  \lim_{m \to \infty} \mE[\mcE(\overline{w}_{t^*(m)})] - \inf_{w\in\mcF} \mcE(w) = 0,
  $$
  provided the sequence $\{\eta_k\}_{k}$ and the stopping rule $t^*(\cdot): \mN \to \mN$ satisfy \\
(A)  $\lim_{m\to \infty} {\sum_{k=1}^{t^*(m)} \eta_k \over m} = 0,$ \\
(B)  and $\lim_{m\to \infty} {1 + \sum_{k=1}^{t^*(m)} \eta_k^2 \over \sum_{k=1}^{t^*(m)}\eta_k} = 0.$
\end{thm}

As seen from the proof in the appendix, Conditions (A) and (B) arise from the analysis of suitable sample,  computational, and approximation errors. Condition (B) is similar to the one required by stochastic gradient methods \cite{bertsekas1999nonlinear,boyd2003subgradient,boyd2007stochastic}. The difference is that here
the limit is taken with respect to the number of points, but the number of passes on the data can be bigger than one.

Theorem \ref{thm:convergence} shows that in order to achieve consistency, the step-sizes and the running iterations need to be appropriately chosen. For instance, given $m$ sample points for SGM with one pass\footnote{We slightly abuse the term `one pass', to mean $m$ iterations.},
 i.e., $t^*(m) = m$, possible choices for the step-sizes are $\{\eta_k= m^{-\alpha}: k\in [m]\}$ and $\{\eta_k= k^{-\alpha}: k\in [m]\}$ for some $\alpha \in (0,1).$ One can also fix the step-sizes {\it a priori}, and then run the algorithm with a suitable stopping rule $t^*(m)$.

These different parameter choices  lead to different implicit  regularization strategies as we  discuss next.

\subsection{Finite Sample Bounds for Smooth Loss Functions}\label{subsec:smoothBound}
In this subsection, we  give explicit finite sample bounds for smooth loss functions, considering  a  suitable assumption on the approximation error.
 \begin{as}\label{as:approximationerror}
 The approximation error associated to the triplet $(\rho, V, \Phi)$ is defined by
\begin{equation}\label{approxerror}
\mathcal{D}(\lambda) = \inf_{w\in \mcF} \left\{ \mathcal{E}(w)  + {\lambda \over 2} \|w\|^2\right\} - \inf_{w\in\mcF}  \mcE(w), \quad \forall \lambda \geq 0.
\end{equation}
We assume that for some $\beta \in (0,1]$ and $c_{\beta}>0$, the approximation error satisfies
\be
\mathcal{D}(\lambda) \leq c_{\beta}\lambda^{\beta}, \qquad \forall \ \lambda> 0.
\label{decayapprox}
\ee
\end{as}

Intuitively,  Condition \eref{decayapprox} quantifies how hard it is to achieve the infimum of the expected risk.
In particular, it is satisfied with $\beta=1$ when\footnote{The existence of at least one minimizer in $\mcF$ is met for example when $\mcF$ is compact, or finite dimensional. In general, $\beta$ does not necessarily have to be 1, since the hypothesis space may be chosen as a general infinite dimensional space, for example in non-parametric regression.
} $\exists w^* \in \mcF$ such that $\inf_{w \in \mcF} \mcE(w) = \mcE(w^*)$.
More formally, the condition  is related to classical terminologies in approximation theory, such as K-functionals and interpolation spaces \cite{steinwart2008support,cucker2007learning}.
The following remark is important for later discussions.
\begin{rem}[SGM and Implicit Regularization]
Assumption \ref{as:approximationerror}
 is standard in statistical learning theory when analyzing  Tikhonov regularization \cite{cucker2007learning,steinwart2008support}. Besides, it has been shown that Tikhonov regularization can achieve best performance by choosing  an appropriate penalty parameter which depends on the unknown parameter $\beta$ \cite{cucker2007learning,steinwart2008support}. In other words, in Tikhonov regularization, the penalty parameter plays a role of regularization.  In this view, our coming results show
that SGM can implicitly implement a form of Tikhonov regularization  by controlling
the step-size and/or the number of passes.
\end{rem}

A further  assumption relates to the smoothness of the loss, and is satisfied for example by the  logistic loss.
\begin{as}\label{as:smooth}
 For all $y\in Y$,  $V(y,\cdot)$ is differentiable and $V'(y,\cdot)$ is Lipschitz continuous with a constant $L>0$, i.e.
$$ |V'(y,b) - V'(y,a)| \leq L|b-a|, \quad \forall a,b\in \mR. $$
\end{as}
The following result characterizes the excess risk of both the last and the average  iterate
for any fixed step-size and stopping time.

\begin{thm}\label{thm:errorSmooth}
  If Assumptions \ref{as:Boundness}, \ref{as:approximationerror} and \ref{as:smooth} hold and  $\eta_t \leq 2/(\kappa^2 L)$ for all $t\in \mN$, then for all $t \in \mN,$
  \begin{align*}
  \mE[\mcE(\overline{w}_t) - \inf_{w\in\mcF}  \mcE(w)] 
   \lesssim   {\sum_{k=1}^t \eta_k \over m} + {\sum_{k=1}^t \eta_k^2 \over \sum_{k=1}^t \eta_k} +  \left( {1 \over \sum_{k=1}^t \eta_k} \right)^{\beta},
  \end{align*}
and
   \begin{align*}
  \mE[\mcE(w_{t}) - \inf_{w\in\mcF}  \mcE(w)] \lesssim {\sum_{k=1}^t \eta_k \over m} \sum_{k=1}^{t-1} \frac{\eta_k}{\eta_t(t-k)}  
   + \left(\sum_{k=1}^{t-1} \frac{\eta_k^2}{\eta_t(t-k)}+ \eta_t\right) + {\left(\sum_{k=1}^t \eta_k\right)^{1-\beta} \over \eta_t t}.
  \end{align*}
\end{thm}
The proof of the above result follows  more or less directly from combining ideas and results in \cite{lin2015iterative,hardt2015train} and is postponed to the appendix.
The constants in the bounds are omitted, but given explicitly in the proof.
While the error bound for the weighted average looks more concise than the one  for the last iterate,  interestingly, both  error bounds  lead to similar generalization properties.

The error bounds are composed of three terms related to sample error, computational error, and approximation error.  Balancing these three error terms to achieve the minimum total error bound leads to optimal choices for the step-sizes $\{\eta_k\}$ and total number of iterations $t^*.$
In other words, both the step-sizes $\{\eta_k\}$ and the number of iterations $t^*$ can play the role of a  regularization parameter. Using the above theorem, general results for step-size $\eta_k = \eta t^{-\theta}$ with some $\theta \in [0,1),\eta= \eta(m)>0$ can be found in Proposition \ref{pro:totalSmooth} from the appendix.
Here, as corollaries we provide four different parameter choices
to obtain the best bounds, corresponding to four different regularization strategies.

The first two corollaries correspond to fixing the step-sizes {\it a priori} and using the number of iterations as a regularization parameter.
In the first result, the step-size is constant and depends on the number of sample points.

\begin{corollary}\label{cor:smoothExplicitC}
If  Assumptions \ref{as:Boundness}, \ref{as:approximationerror} and \ref{as:smooth} hold and
  $\eta_t = \eta_1/\sqrt{m}$  for all $t\in \mN$
 for some  positive constant $\eta_1 \leq 2/(\kappa^2L)$,  then for all $t \in \mN,$ and $g_t = \overline{w}_t$ (or $w_t$),
\be\label{smoothExplicitC}
  \mE [\mcE(g_t) - \inf_{w\in\mcF}  \mcE(w)]
  \lesssim { t \log t  \over \sqrt{m^3} } +  { \log t \over \sqrt{m}} + \left({ \sqrt{m} \over t}\right)^{\beta} .
  \ee
In particular, if we choose $t^* = \lceil m^{\beta+3 \over 2(\beta+1)} \rceil,$
\be\label{optimalboundsSmooth}
  \mE[\mcE(g_{t^*}) - \inf_{w\in\mcF}  \mcE(w)] \lesssim   m^{-{\beta \over \beta+1}}\log m .
  \ee
\end{corollary}
In the second result the step-sizes decay with the iterations.

\begin{corollary}\label{cor:smoothExplicit}
  If Assumptions \ref{as:Boundness}, \ref{as:approximationerror} and \ref{as:smooth} hold and
  $\eta_t = \eta_1/\sqrt{t}$  for all $t\in \mN$
  with some  positive constant $\eta_1 \leq 2/(\kappa^2L)$, then for all $t \in \mN,$ and $g_t = \overline{w}_t$ (or $w_t$),
\be\label{smoothExplicit}
  \mE [\mcE(g_t) - \inf_{w\in\mcF}  \mcE(w)]
  \lesssim { \sqrt{t} \log t  \over m} +  { \log t\over \sqrt{t}} + { 1 \over t^{\beta / 2}} .
  \ee
Particularly, when $t^* = \lceil m^{2 \over \beta+1} \rceil,$ we have \eref{optimalboundsSmooth}.
\end{corollary}

In both the above  corollaries  the step-sizes are fixed {\it a priori}, and  the  number of iterations becomes the  regularization parameter controlling the total error. Ignoring the logarithmic factor, the dominating terms in the bounds  \eref{smoothExplicitC}, \eref{smoothExplicit} are the sample  and approximation errors, corresponding to the first and third terms of RHS. Stopping too late may lead to a  large sample error, while stopping too early may lead to a large approximation error. The ideal stopping time arises from a form of bias-variance trade-off and requires in general more than one pass over the data. Indeed, if we reformulate the results in terms of number of passes, we have that   $\lceil m^{1 - \beta \over 2(1 + \beta)} \rceil$ passes are needed for the constant step-size $\{\eta_t =\eta_1 / \sqrt{m} \}_t$, while
 $\lceil m^{1 - \beta \over 1 + \beta} \rceil$ passes are needed for the  decaying step-size  $\{\eta_t =\eta_1 / \sqrt{t} \}_t$. These observations suggest in particular that  while both step-size choices achieve the same bounds, the constant step-size can have a computational advantage since it requires less iterations.

 Note that one pass over the data suffices only  in the limit case when $\beta=1$, while in general it will be suboptimal, at least if the step-size is fixed. In fact,
Theorem \ref{thm:errorSmooth} suggests that  optimal results could be recovered if the step-size is suitably tuned. The next corollaries show that this is indeed the case.
The first result corresponds to a suitably tuned constant step-size.
\begin{corollary}\label{cor:smoothExplicitCSingle}
  If Assumptions \ref{as:Boundness}, \ref{as:approximationerror} and \ref{as:smooth} hold and  $\eta_t = \eta_1 m^{-{\beta \over \beta+1}}$  for all $t\in \mN$  for some  positive constant $\eta_1 \leq 2/(\kappa^2L)$, then for all $t \in \mN,$ and $g_t = \overline{w}_t$ (or $w_t$),
\begin{align*}
  \mE [\mcE(g_t) - \inf_{w\in\mcF}  \mcE(w)] 
  \lesssim  { m^{-{\beta +2 \over \beta+1}} t \log t   } +  { m^{-{\beta \over \beta+1}} \log t } + m^{{\beta^2 \over \beta+1}} t^{-\beta} .
  \end{align*}
In particular, we have \eref{optimalboundsSmooth} for $t^*=m.$
\end{corollary}

	The second result corresponds to tuning the decay rate for a decaying step-size.
	
\begin{corollary}\label{cor:smoothExplicitDSingle}
If Assumptions \ref{as:Boundness}, \ref{as:approximationerror} and \ref{as:smooth}
hold and  $\eta_t = \eta_1 t^{-{\beta \over \beta+1}}$  for all $t\in \mN$  for some  positive constant $\eta_1 \leq 2/(\kappa^2L)$,  then for all $t \in \mN,$ and $g_t = \overline{w}_t$ (or $w_t$),
\begin{align*}
  \mE [\mcE(g_t) - \inf_{w\in\mcF}  \mcE(w)] 
  \lesssim   m^{-1} { t^{{1 \over \beta+1}} \log t }  +  { t^{-{\beta \over \beta+1}} \log t } + { t^{-{\beta \over \beta+1}}} .
  \end{align*}
  In particular, we have \eref{optimalboundsSmooth} for $t^* = m.$
\end{corollary}

The above two results confirm that good performances can be attained with only one pass over the data, provided the step-sizes are suitably chosen, that is using  the step-size
as a regularization parameter.

\begin{rem}
  If we further assume that $\beta=1,$ as often done in the literature, the convergence rates from Corollaries 1-4 are of order $O(m^{-1/2}),$
  which are the same as those in, e.g., \cite{shamir2013stochastic}.
\end{rem}

Finally,  the following remark relates the above results to data-driven parameter tuning used in practice.
\begin{rem}[Bias-Variance and Cross-Validation]
The above results show how the number of iterations/passes controls a bias-variance trade-off, and in this sense acts as a regularization parameter. In practice, the approximation properties of the algorithm are unknown and the question arises of how the parameter can be chosen.  As it turns out, cross-validation can be used to achieve adaptively the best rates, in the sense that the rate in \eref{optimalboundsSmooth} is achieved by cross-validation or more precisely by hold-out cross-validation.
These results follow by an argument similar to that in Chapter 6 from \cite{steinwart2008support} and are omitted.
\end{rem}

\subsection{Finite Sample Bounds for Non-smooth Loss Functions}\label{subsec:nonsmooth}
Theorem \ref{thm:errorSmooth} holds for smooth loss functions and it is natural to ask if
a similar result holds for non-smooth losses such as the hinge loss.
Indeed, analogous results hold, albeit current bounds are not as sharp.
\begin{thm}\label{thm:errorGeneral}
  If Assumptions \ref{as:Boundness} and \ref{as:approximationerror} hold, then $\forall t \in \mN,$
  \begin{align*}
  \mE[\mcE(\overline{w}_t) - \inf_{w\in\mcF}  \mcE(w)] 
  \lesssim   \sqrt{\sum_{k=1}^t \eta_k \over m} + {\sum_{k=1}^t \eta_k^2 \over \sum_{k=1}^t \eta_k} +  \left( {1 \over \sum_{k=1}^t \eta_k} \right)^{\beta} ,
  \end{align*}
  and
  \begin{align*}
  \mE[\mcE(w_{t}) - \inf_{w\in\mcF}  \mcE(w)]  \lesssim\sqrt{\sum_{k=1}^t \eta_k \over m}  \sum_{k=1}^{t-1} \frac{\eta_k}{\eta_t(t-k)} 
  + \sum_{k=1}^{t-1} \frac{\eta_k^2}{\eta_t(t-k)} + \eta_t + {\left(\sum_{k=1}^t \eta_k\right)^{1-\beta} \over \eta_t t} .
  \end{align*}
\end{thm}

The proof of the above theorem is based on ideas from \cite{lin2015iterative}, where tools from Rademacher complexity \cite{bartlett2003rademacher,meir2003generalization} are employed.
We postpone the proof in the appendix.

Using the above result with concrete step-sizes as those for smooth loss functions, we have the following explicit error bounds and corresponding stopping rules.
\begin{corollary}\label{cor:genExplicitC}
    Under Assumptions \ref{as:Boundness} and \ref{as:approximationerror}, let $\eta_t = 1/\sqrt{m}$  for all $t\in \mN$. Then for all $t \in \mN,$ and $g_t = \overline{w}_t$ (or $w_t$),
\begin{align*}
  \mE [\mcE(g_t) - \inf_{w\in\mcF}  \mcE(w)] \lesssim { \sqrt{t} \log t  \over m^{3/4} } +  { \log t \over \sqrt{m}} + \left({ \sqrt{m} \over t}\right)^{\beta} .
  \end{align*}
In particular, if we choose $t^* = \lceil m^{2\beta+3 \over 4\beta+2} \rceil,$
\be\label{optimalboundsNonSmooth}
  \mE[\mcE(g_{t^*}) - \inf_{w\in\mcF}  \mcE(w)] \lesssim   m^{-{\beta \over 2\beta+1}} \log m .
  \ee
\end{corollary}

\begin{corollary}\label{cor:genExplicit}
   Under Assumptions \ref{as:Boundness} and \ref{as:approximationerror},  let $\eta_t = 1/\sqrt{t}$  for all $t\in \mN$. Then for all $t \in \mN,$ and $g_t = \overline{w}_t$ (or $w_t$),
\bea
  \mE [\mcE(g_t) - \inf_{w\in\mcF}  \mcE(w)] \lesssim { t^{1/4}\log t \over \sqrt{m}}+ { \log t\over \sqrt{t}} + { 1 \over t^{\beta / 2}} .
\eea
In particular, if we choose $t^* = \lceil m^{2 \over 2\beta+1} \rceil,$ there holds \eref{optimalboundsNonSmooth}.
\end{corollary}
From the above two corollaries, we  see that the algorithm with constant step-size $1/\sqrt{m}$ can stop earlier
than the one with decaying step-size $1/\sqrt{t}$ when $\beta\leq 1/2,$ while they have the same convergence rate,
since $m^{2\beta+3 \over 4\beta+2} / m^{2 \over 2\beta+1} = m^{2\beta-1 \over 4\beta+1}.$
Note that the bound in \eref{optimalboundsNonSmooth} is slightly worse than that in \eref{optimalboundsSmooth}, see Section \ref{sec:discussion} for more discussion.

Similar to the smooth case, we also have the following results for SGM with one pass where regularization is realized by step-size.

\begin{corollary}\label{cor:genExplicitCSingle}
   Under Assumptions \ref{as:Boundness} and \ref{as:approximationerror}, let $\eta_t = m^{-{2\beta \over 2\beta+1}}$  for all $t\in \mN$. Then for all $t \in \mN,$ and $g_t = \overline{w}_t$ (or $w_t$),
\begin{align*}
  \mE [\mcE(g_t) - \inf_{w\in\mcF}  \mcE(w)] 
  \lesssim { m^{-{4\beta+1 \over 4\beta+2}}\sqrt{t} \log t  } +  m^{-{2\beta \over 2\beta+1}} \log t  + m^{{2\beta^2 \over 2\beta+1}} t^{-\beta}.
  \end{align*}
In particular, \eref{optimalboundsNonSmooth} holds for $t^* = m.$
\end{corollary}

\begin{corollary}\label{cor:genExplicitSingle}
   Under Assumptions \ref{as:Boundness}  and \ref{as:approximationerror}, let $\eta_t = t^{-{2\beta \over 2\beta+1}}$  for all $t\in \mN$. Then for all $t \in \mN,$ and $g_t = \overline{w}_t$ (or $w_t$),
\begin{align*}
  \mE [\mcE(g_t) - \inf_{w\in\mcF}  \mcE(w)] 
  \lesssim m^{- {1 \over 2}}t^{1 \over 4\beta+2}\log t   + { t^{-{\min(2\beta,1) \over 2\beta+1}} \log t} + t^{-{\beta \over 2\beta+1}}  .
\end{align*}
In particular, \eref{optimalboundsNonSmooth} holds for $t^* = m.$
\end{corollary}

\subsection{Discussion and Proof Sketch}
\label{sec:discussion}

As mentioned in the introduction,  the literature on theoretical  properties of the iteration in Algorithm ~\ref{alg:SIGD}
is vast, both in learning theory and in optimization. A first line of works focuses on a single pass and convergence
of the expected risk. Approaches in this sense include classical results in optimization (see \cite{nemirovski2009robust} and references therein), but also approaches based on so-called ``online to batch" conversion (see \cite{orab14} and references therein). The latter are based on analyzing a  sequential prediction setting  and then on considering the averaged iterate to turn regret bounds in expected risk bounds. A second line of works focuses on multiple passes, but measures the quality of the corresponding iteration in terms of the minimization of the empirical risk.
In this view, Algorithm~\ref{alg:SIGD} is seen as an instance of incremental methods for the minimization of objective functions that are sums of a finite, but possibly large, number of terms \cite{bertsekas2011incremental}.  These latter works, while interesting in their own right, do not yield any direct information on the generalization properties of considering multiple passes.

Here,  we follow the approach in \cite{bousquet2008tradeoffs} advocating the combination of statistical and computational errors. The general proof strategy is to consider several intermediate steps to relate the expected risk of the empirical iteration to the minimal expected risk.
The argument we sketch below is a simplified and less sharp version with respect to the one used in the actual proof,
but it is easier to illustrate and still carries some important aspects which are useful for comparison with related results.

Consider an intermediate element $\tilde w\in \mcF$ and decompose the excess risk as
\begin{eqnarray*}
 {\mathbb E}
\mcE(w_t)-\inf_{w\in \mcF}\mcE =
 {\mathbb E} (\mcE(w_t)- \mcE_{\mathbf{z}}(w_t))
+{\mathbb E}(\mcE_{\mathbf{z}}(w_t)-\mcE_{\mathbf{z}}(\tilde w))
+{\mathbb E}\mcE_{\mathbf{z}}(\tilde w)-\inf_{w\in \mcF} \mcE.
\end{eqnarray*}
The first term on the right-hand side is the generalization error of the iterate. The second term
can be seen as a computational error. To discuss the last term, it is useful to consider a few different choices for  $\tilde w$. Assuming the empirical and expected risks to have  minimizers $w^*_{\mathbf{z}}$ and $w^*$,  a possibility is to set $\tilde w=w^*_{\mathbf{z}}$, this can be seen to be  the choice made in \cite{hardt2015train}. In this case, it is immediate to see that the last term is negligible since,
$$
{\mathbb E}\mcE_{\mathbf{z}}(\tilde w)={\mathbb E} \min_{w\in \mcF}\mcE_{\mathbf{z}}( w) \le
\min_{w\in \mcF} {\mathbb E}\mcE_{\mathbf{z}}( w)
= \min_{w\in \mcF} \mcE( w),
$$
and hence,
$$
{\mathbb E}\mcE_{\mathbf{z}}(\tilde w)  -\min_{w\in \mcF} \mcE\le 0.
$$
On the other hand,  in this case the computational error depends on the norm $\|w^*_{\mathbf{z}}\|$ which is in general hard  to estimate.   A more convenient choice  is to set $\tilde w=w^*$. A reasoning similar to the one above shows that the last term is still  negligible and the computational error  can still be controlled depending on $\|w^*\|$. In a non-parametric setting, the existence of a minimizer is not ensured and corresponds to a limit case where there is small approximation error.
Our approach  is then to consider an {\em almost} minimizer of the expected risk with a prescribed accuracy.
Following \cite{lin2015iterative}, we do this introducing Assumption~\eqref{approxerror} and
choosing $\tilde w $ as the unique minimizer of $\mcE+\lambda\|\cdot\|^2$, $\lambda>0$.
Then the last term in the error decomposition  can be upper bounded by the approximation error.

For the generalization error,  the stability results from \cite{hardt2015train} provide sharp estimates
for  smooth loss functions and in  the `capacity independent' limit, that is under no assumptions on the covering numbers of the considered function space. For this setting, the obtained bound is optimal in the sense that it  matches the best available bound for Tikhonov regularization \cite{steinwart2008support,cucker2007learning}.   For the non-smooth case a standard argument based on Rademacher complexity can be used, and easily extended to be capacity dependent. However, the corresponding bound is not sharp and improvements are likely to hinge on
deriving  better norm estimates  for  the iterates. The question does not seem to be  straightforward and is deferred to a future work.

The computational error for the averaged iterates can be controlled using classic arguments \cite{boyd2007stochastic}, whereas for the last iterate the arguments in \cite{lin2015iterative,shamir2013stochastic} are needed. Finally, Theorems~\ref{thm:errorSmooth}, \ref{thm:errorGeneral}  result from estimating and balancing the various error terms with respect to the choice of the step-size and number of passes.

We conclude this section with some perspective on the results in the paper.
We note that  since  the primary goal of this study was to  analyze the implicit  regularization effect of step-size and number of passes, we have considered a very simple iteration.  However, it would be very interesting to consider  more sophisticated, `accelerated' iterations \cite{schmidt2013minimizing}, and assess the potential advantages in terms of computational and generalization aspects.
Similarly, we chose to keep the analysis in the paper relatively simple, but several improvements can be considered
for example deriving high probability bounds and sharper error bounds under further assumptions. Some of these improvements are relatively straightforward,
see e.g.  \cite{lin2015iterative}, but others will require non-trivial extensions of results developed for Tikhonov regularization in the last few years. Finally, here we only referred to  a simple cross-validation approach to parameter tuning, but
it would clearly be very interesting to find ways to tune parameters  online. A remarkable result in this direction
is derived in  \cite{orab14},  where it is shown that, in the capacity independent setting, adaptive online parameter tuning is indeed possible.

\section{Numerical Simulations} \label{sec:simulations}

\begin{table}
\caption{Benchmark datasets and Gaussian kernel width $\sigma$ used in our experiments.}
\begin{center}
\begin{adjustbox}{max width=0.48\textwidth, max totalheight=0.9\textheight}
\begin{tabular}{M{3cm}M{1.3cm}M{1.3cm}M{1.3cm}M{1.3cm}N}
\toprule
{\em Dataset} & $n$ &  $n_{test}$ & $d$ & $\sigma$ &\\
\midrule
{\em BreastCancer} & 400 & 169 & 30 & 0.4 &\\
{\em Adult} & 32562 & 16282 & 123 &    4 &\\
{\em Ijcnn1} & 49990 & 91701 & 22 &      0.6 &\\
\bottomrule\hline
\end{tabular}
\end{adjustbox}
\end{center}
\label{tab:dataSpec}
\end{table}

We carry out some numerical simulations to illustrate our results\footnote{Code:  \texttt{lcsl.github.io/MultiplePassesSGM}}.
The experiments are executed 10 times each, on the benchmark datasets\footnote{Datasets: \texttt{archive.ics.uci.edu/ml} and \texttt{www.csie.ntu.edu.tw/\texttildelow cjlin/libsvmtools/} \texttt{datasets/}} reported in Table \ref{tab:dataSpec}, in which the Gaussian kernel bandwidths $\sigma$ used by SGM and SIGM\footnote{In what follows, we name one pass SGM and multiple passes SGM as SGM and SIGM, respectively.}
 for each learning problem are also shown. Here, the loss function is the hinge loss\footnote{Experiments with the logistic loss have also been carried out,
showing similar empirical results to those
considering the hinge loss. The details are not included in this text due to space limit.
}.
The experimental platform is a server with 12 $\times$ Intel$^\circledR$ Xeon$^\circledR$ E5-2620 v2 (2.10GHz) CPUs and 132 GB of RAM.
Some of the experimental results, as specified in the following, have been obtained by running the experiments on subsets of the data samples chosen uniformly at random.
In order to apply hold-out cross-validation, the training set is split in two parts: one for empirical risk minimization and the other for validation error computation (80\% - 20\%, respectively).
All the samples are randomly shuffled at each repetition.

\subsection{Regularization in SGM and SIGM}

\begin{figure}[t!]
    \centering
    \subfigure[]{ 	\label{fig:Adult1000_SIGD1}\includegraphics[width=0.45\textwidth]{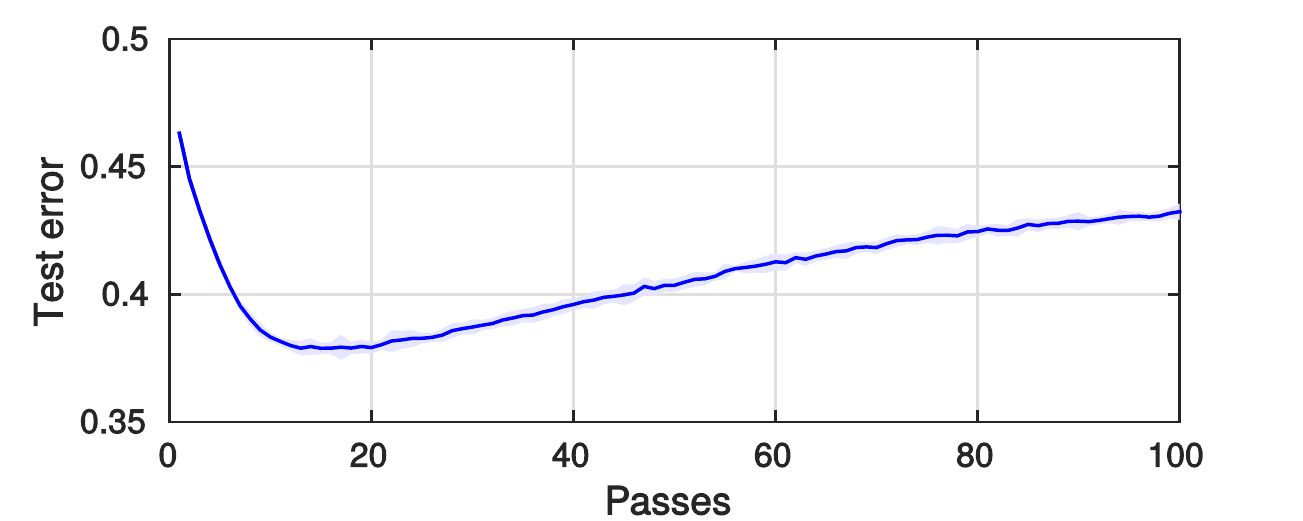}
    }
    ~ 
    \subfigure[]{	\label{fig:Adult1000_SIGD2}\includegraphics[width=0.45\textwidth]{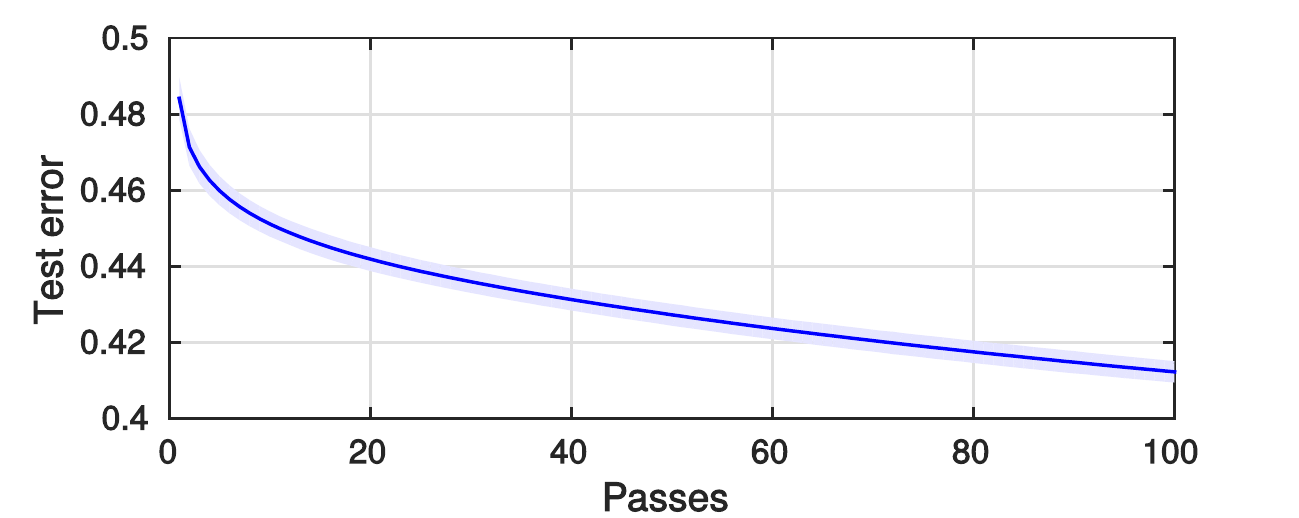}
    }
    \caption{Test error for SIGM with fixed (a) and decaying (b) step-size with respect to the number of passes on {\em Adult} ({\em n = 1000}).}
    \label{fig:Adult1000_SIGD}
\end{figure}

\begin{figure}[t!]
    \centering
    \subfigure[]{	\label{fig:Adult1000_SGD1}\includegraphics[width=0.45\textwidth]{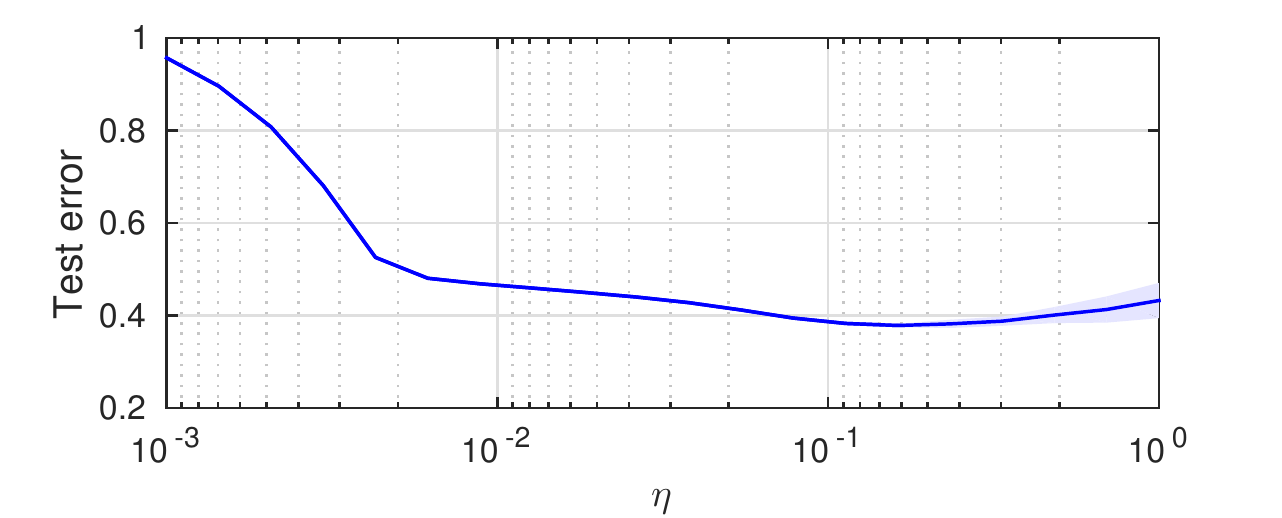}
    }
    ~ 
    \subfigure[]{	\includegraphics[width=0.45\textwidth]{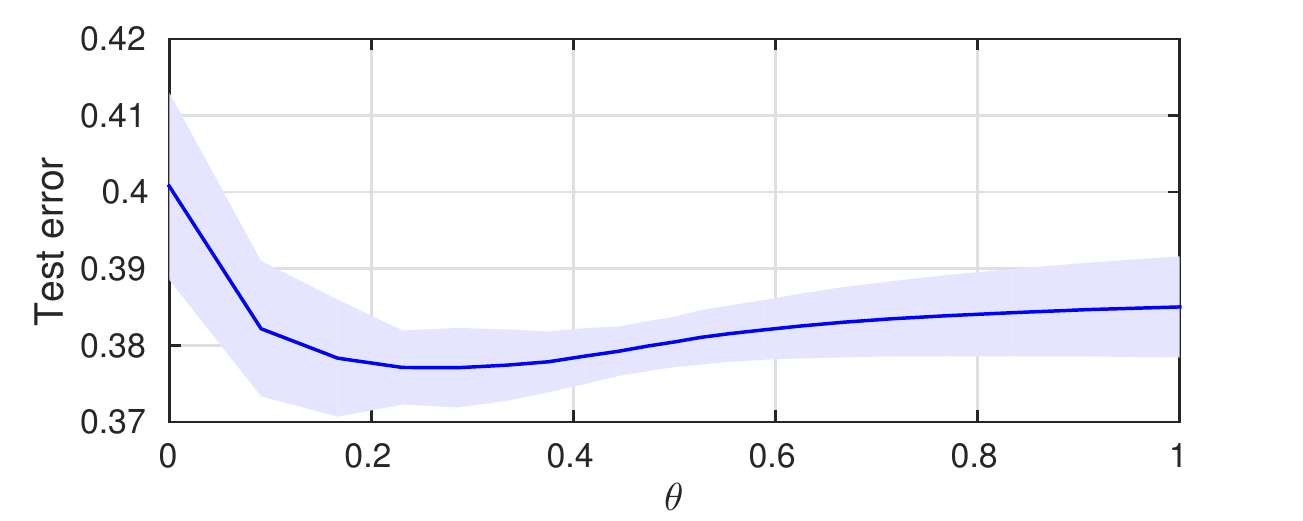}\label{fig:Adult1000_SGD2}
    }
    \caption{Test error for SGM with fixed (a) and decaying (b) step-size cross-validation on {\em Adult} ({\em n = 1000}).}
	\label{fig:Adult1000_SGD}
\end{figure}

\label{sec:simulation_regularization}
In this subsection, we illustrate four concrete examples showing different regularization effects of the step-size in SGM and the number of passes in SIGM. In all these four examples, we consider the {\em Adult} dataset with sample size $n = 1000$. 

In the first experiment, the SIGM step-size is fixed as $\eta = 1/\sqrt{n}$.
 The test error computed with respect to the hinge loss at each pass is reported in Figure \ref{fig:Adult1000_SIGD1}.
Note that the minimum test error is reached for a number of passes smaller than 20, after which it significantly increases, a so-called overfitting regime. This result clearly illustrates the regularization effect of the number of passes.
In the second experiment, we
consider SIGM with decaying step-size ($\eta = 1/4$ and $\theta = 1/2$).
As shown in Figure \ref{fig:Adult1000_SIGD2}, overfitting is not observed in the first 100 passes. In this case, the convergence to the optimal solution appears slower than that in the fixed step-size case.

In the last two experiments, we consider SGM and show that the step-size plays the role of a regularization parameter.
 For the  fixed step-size case, i.e., $\theta = 0$, we perform SGM with different $ \eta \in ( 0, 1 ]$ (logarithmically scaled).
We plot the errors in Figure \ref{fig:Adult1000_SGD1}, showing that a large step-size ($\eta = 1$) leads to overfitting, while a smaller one (e.g., $\eta = 10^{-3}$) is associated to oversmoothing.
 For the decaying step-size case, we fix $\eta_1 = 1/4$, and run SGM with different $\theta \in [0,1]$. The errors are plotted in Figure \ref{fig:Adult1000_SGD2}, from which we see that the exponent $\theta$ has a regularization effect. In fact, a more `aggressive' choice (e.g., $\theta = 0$, corresponding to a fixed step-size) leads to overfitting, while for a larger $\theta$ (e.g., $\theta = 1$) we observe oversmoothing.

\subsection{Accuracy and Computational Time Comparison}
\label{sec:experiments_accuracy}
\begin{table}
\caption{Comparison of SGM and SIGM with cross-validation with decaying (D) and constant (C) step-sizes, in terms of computational time and accuracy. SGM performs cross-validation on 30 candidate step-sizes, while SIGM achieves implicit regularization via early stopping.}
\begin{center}
\begin{adjustbox}{max width=0.6\textwidth, max totalheight=0.8\textheight}
\begin{tabular}{M{1.6cm}M{1.5cm}M{0.8cm}M{2.1cm}M{2.1cm}M{2.1cm}N}
\toprule
{\em Dataset} & {\em Algorithm} & {\em Step Size} & \specialcell{{\em Test Error}\\ {\em (hinge loss)}} &\specialcell{{\em Test Error}\\ {\em (class. error)}} & \specialcell{{\em Training}\\ {\em Time (s)}} & \\
\midrule
\multirow{2}{*}{\specialcell{ \\ {\em BreastCancer}\\ $ n = 400 $ }}
 & $SGM$ & C &  $0.127 \pm 0.022$ & $3.1 \pm 1.1\%$  & 1.7 $\pm$ 0.2 &   \\
 & $SGM$ & D &  $0.135 \pm 0.024$ & $3.0 \pm 1.1\%$  & 1.4 $\pm$ 0.3 &   \\
 & $SIGM$ & C & 0.131 $\pm$ 0.023 & $3.2 \pm 1.1\%$ & $1.4 \pm 0.8$ & \\
 & $SIGM$ & D & 0.204 $\pm$ 0.017 & $3.9 \pm 1.0\%$ & $1.8 \pm 0.5$ & \\
 & $LIBSVM$ & & & $2.8 \pm 1.3\%$ & $0.2 \pm 0.0$ & \\
 \hline
\multirow{2}{*}{\specialcell{ \\ {\em Adult}\\ $ n = 1000 $ }}
 & $SGM$ & C& 0.380 $\pm$ 0.003 & $16.6 \pm 0.3\%$ & 5.7 $\pm$ 0.6 & \\
 & $SGM$ & D&  0.378 $\pm$ 0.002 & $16.2 \pm 0.2\%$ & 5.4$\pm$ 0.3 & \\
 & $SIGM$ & C & $0.383 \pm 0.002$ & $16.1 \pm 0.0\%$& $3.2 \pm 0.4$ &\\
 & $SIGM$ & D & $0.450 \pm 0.002$ & $23.6 \pm 0.0\%$& $1.6 \pm 0.2$ &\\
 & $LIBSVM$ & & & $18.7 \pm 0.0\%$ & $5.8\pm 0.5$ & \\
 \hline
\multirow{2}{*}{\specialcell{ \\ {\em Adult}\\ $ n = 32562 $ }}
 & $SGM$ & C & 0.342 $\pm$ 0.001 & $15.2 \pm 0.8 \%$ & 320.0 $\pm$ 3.3 &   \\
 & $SGM$ & D & 0.340 $\pm$ 0.001 & $15.1 \pm 0.7 \%$ & 332.1 $\pm$ 3.3 &   \\
 & $SIGM$ & C & $0.343 \pm 0.001$ & $15.7 \pm 0.9 \%$& $366.2 \pm 3.9$ & \\
 & $SIGM$ & D& $0.364 \pm 0.001$ & $17.1 \pm 0.8 \%$& $442.4 \pm 4.2$ &    \\
 & $LIBSVM$ & &  & $ 15.3 \pm 0.7\%$ & $ 6938.7\pm 171.7$ & \\
 \hline
\multirow{2}{*}{\specialcell{ \\ {\em Ijcnn1}\\ $ n = 1000 $ }}
 & $SGM$ & C & 0.199 $\pm$ 0.016 & $8.4 \pm 0.8 \%$ & 3.9 $\pm$ 0.3 &  \\
 & $SGM$ & D & 0.199 $\pm$ 0.009 & $9.1 \pm 0.1 \%$ & 3.8$\pm$ 0.3 &  \\
 & $SIGM$ & C & $0.205 \pm 0.010$ & $9.3 \pm 0.5 \%$ & $1.7 \pm 0.4$ &  \\
 & $SIGM$ & D & $0.267 \pm 0.006$ & $9.4\pm 0.6 \%$ & $2.2 \pm 0.4$ &  \\
 & $LIBSVM$ & &  & $7.1 \pm 0.7\%$ & $0.6 \pm 0.1$ & \\
 \hline
\multirow{2}{*}{\specialcell{ \\ {\em Ijcnn1}\\ $ n = 49990 $ }}
 & $SGM$ & C & $0.041 \pm 0.002$ & $ 1.5 \pm 0.0 \%$ & $564.9 \pm 6.3$ &\\
 & $SGM$ & D & $0.059 \pm 0.000$ & $ 1.7 \pm 0.0 \%$ & $578.9 \pm 1.8$ &\\
 & $SIGM$ & C & 0.098 $\pm$ 0.001 & $ 4.7 \pm 0.1\%$ & 522.2 $\pm$ 20.7 & \\
 & $SIGM$ & D & 0.183 $\pm$ 0.000 & $ 9.5 \pm 0.0\%$ & 519.3 $\pm$ 25.8 & \\
 & $LIBSVM$ & &  & $ 0.9 \pm 0.0\%$ & $ 770.4 \pm 38.5$ & \\
 \bottomrule
 \hline
\end{tabular}
\label{tab:testSetComparison}
\end{adjustbox}
\end{center}
\end{table}
In this subsection, we compare SGM with cross-validation and SIGM with benchmark algorithm LIBSVM~\cite{CC01a}, both in terms of accuracy and computational time. For SGM, with 30 parameter guesses, we use cross-validation to tune the step-size (either setting $\theta=0$ while tuning $\eta$, or setting $\eta = 1/4$ while tuning $\theta$). For SIGM, we use two kinds of step-size suggested by Section \ref{sec:theory}: $\eta= 1/\sqrt{m}$ and $\theta=0$, or $\eta=1/4$ and $\theta=1/2,$ using early stopping via cross-validation.
  The test errors with respect to the hinge loss, the test relative misclassification errors and the computational times are collected in Table \ref{tab:testSetComparison}.

We first start comparing accuracies. The results in Table \ref{tab:testSetComparison} indicate that SGM with constant and decaying step-sizes and SIGM with fixed step-size reach comparable test errors, which are in line with the LIBSVM baseline.
 Observe that SIGM with decaying step-size attains consistently higher test errors, a phenomenon already illustrated in Section \ref{sec:simulation_regularization} in theory.

We now compare the computational times for cross-validation.
We see from Table \ref{tab:testSetComparison} that the training times of SIGM and SGM, either with constant or decaying step-sizes, are roughly the same.
We also observe that SGM and SIGM are faster than LIBSVM  on relatively large datasets ({\em Adult with n = 32562}, and {\em Ijcnn1 with n = 49990}). Moreover, for small datasets ({\em BreastCancer} with {\em n = 400}, {\em Adult} with {\em n = 1000}, and {\em Ijcnn1} with {\em n = 1000}), SGM and SIGM are comparable with or slightly slower than LIBSVM.

\section*{Acknowledgments}
 This material is based upon work supported by the Center for Brains, Minds and Machines (CBMM), funded by NSF STC award CCF-1231216. L. R. acknowledges the financial support of the Italian Ministry of Education, University and Research FIRB project RBFR12M3AC.
  The authors would like to thank Dr. Francesco Orabona for the fruitful discussions on this research topic, and Dr. Silvia Villa and the referees for their valuable comments.

\bibliographystyle{plain}

\renewcommand\appendixpagename{Appendices: Proofs}
\appendix
\appendixpage

\section{Basic Lemmas}

The following basic lemma is useful to our proofs, which will be used several times.
Its proof follows from the convexity of $V(y,\cdot)$ and the fact that $V_-'(y,a)$ is bounded.
\begin{lemma}\label{lemma:general}
 Under Assumption \ref{as:Boundness}, for any $k \in \mN$ and $w\in \mcF$, we have
\be\label{generalInequality}
   \|w_{k+1}-w\|^2 \leq \|w_k-w\|^2 + (a_0\kappa)^2 \eta_k^2 + 2\eta_{k} \left[ V(y_{j_k}, \la w, \Phi(x_{j_k})\ra) - V(y_{j_k}, \la w_k, \Phi(x_{j_k})\ra )\right].
\ee
\end{lemma}
\begin{proof}
 Since $w_{k+1}$ is given by \eref{SIGD}, by expanding the inner product, we have
   $$
   \|w_{k+1}-w\|^2=\|w_k-w\|^2 + \eta_k^2\|V'_- (y_{j_k}, \la w_k, \Phi(x_{j_k}) \ra) \Phi(x_{j_k})\|^2 + 2\eta_{k}V'_{-}(y_{j_k}, \la w_k, \Phi(x_{j_k}) \ra ) \left\la w-w_k, \Phi(x_{j_k})\right\ra.
  $$
The bounded assumption \eref{boundedkernel} implies that
$\|\Phi(x_{j_k})\| \leq \kappa $ and by  \eref{boundedDeriviative}, $|V'_{-}(y_{j_k}, \la w_k, \Phi(x_{j_k}) \ra )| \leq a_0.$
We thus have
$$
   \|w_{k+1}-w\|^2 \leq \|w_k-w\|^2 + (a_0\kappa)^2 \eta_k^2 + 2\eta_{k}V'_{-}(y_{j_k}, \la w_k, \Phi(x_{j_k}) \ra ) [ \left\la w, \Phi(x_{j_k})\right\ra - \left\la w_k, \Phi(x_{j_k})\right\ra ].
  $$
Using
the convexity of $V (y_{j_k}, \cdot)$ which tells us that
$$V'_- (y_{j_k}, a) (b-a)\leq V (y_{j_k}, b)-V (y_{j_k}, a), \qquad \forall a, b\in\mR,$$
 we reach the desired bound. The proof is complete.
\end{proof}

Taking the expectation of \eref{generalInequality} with respect to the random variable $j_k$, and noting that $w_{k}$ is independent from $j_k$ given ${\bf z}$, one can get the following result.
\begin{lemma}\label{lemma:empiricalGF}
  Under Assumption \ref{as:Boundness},  for any fixed $k \in \mN,$ given any ${\bf z}$, assume that $w\in \mcF$ is independent of the random variable $j_k$. Then we have
  \be\label{empiricalGF}
  \mE_{j_{k}} [\|w_{k+1} - w\|^2] \leq \|w_k - w\|^2 + (a_0 \kappa)^2 \eta_k^2  + 2\eta_k \left( \mcE_{\bf z}(w) - \mcE_{\bf z}(w_k)\right).
  \ee
\end{lemma}
\section{Sample Errors}
Note that our goal is to bound the excess generalization error $\mE[\mcE(w_{T}) - \inf_{w\in \mcF} \mcE(w)],$ whereas the left-hand side of \eref{empiricalGF} is related to an empirical error.
The difference between the generalization and empirical errors is a so-called sample error.
To estimate this sample error, we introduce the following lemma, which gives a uniformly upper bound for sample errors over a ball $B_R = \{w\in \mcF:\|w\| \leq R\}$. Its proof is
based on a standard symmetrization technique and Rademacher complexity, e.g. \cite{bartlett2005local,meir2003generalization}.
For completeness, we provide a proof here.
\begin{lemma}
  \label{lemma:UniformSampleErrors}
 Assume \eref{boundedkernel} and \eref{boundedDeriviative}. For any $R>0,$ we have
  \bea
  \left|\mE_{\bf z} \left[\sup_{w \in B_R}(\mcE(w) - \mcE_{\bf z}(w))\right]\right|
  \leq {2a_0 \kappa R \over \sqrt{m}}.
  \eea
\end{lemma}
\begin{proof}
 Let ${\bf z'} = \{z_i' = (x_i',y'_i)\}_{i=1}^{m}$ be another training sample from $\rho$, and assume that it is independent from $\bf z.$
  We have
  \bea
  \mE_{\bf z} \left[\sup_{w \in B_R}(\mcE(w) - \mcE_{\bf z}(w))\right] = \mE_{\bf z} \left[\sup_{w \in B_R}\mE_{\bf z'}[\mcE_{\bf z'}(w) - \mcE_{\bf z}(w)]\right]
  \leq \mE_{\bf z,z'} \left[\sup_{w \in B_R}(\mcE_{\bf z'}(w) - \mcE_{\bf z}(w))\right].
  \eea
  Let $ \sigma_1, \sigma_2, \dots, \sigma_m$ be independent random variables drawn from the Rademacher distribution, i.e. $\Pr(\sigma_i = +1) = \Pr(\sigma_i = -1) = 1/2$ for $i=1,2,\dots,m$.
   Using a standard symmetrization technique, for example in \cite{meir2003generalization}, we get
   \bea
  \mE_{\bf z} \left[\sup_{w \in B_R}(\mcE(w) - \mcE_{\bf z}(w))\right]
  &\leq& \mE_{\bf z,z', \sigma} \left[\sup_{w \in B_R}\left\{ {1\over m} \sum_{i=1}^m \sigma_{i} (V(y_i',\la w, \Phi(x_i')\ra ) - V(y_i,\la w, \Phi(x_i)\ra ) )\right\}\right]\\
  &\leq& 2 \mE_{\bf z, \sigma} \left[\sup_{w \in B_R}\left\{ {1\over m} \sum_{i=1}^m \sigma_{i} V(y_i,\la w, \Phi(x_i)\ra ) \right\}\right].
  \eea
 With  \eref{boundedDeriviative}, by applying Talagrand's contraction lemma, see e.g. \cite{bartlett2005local}, we derive
\bea
  \mE_{\bf z} \left[\sup_{w \in B_R}(\mcE(w) - \mcE_{\bf z}(w))\right]
  \leq 2a_0 \mE_{\bf z, \sigma} \left[\sup_{w \in B_R} {1\over m} \sum_{i=1}^m \sigma_{i} \la w, \Phi(x_i) \ra \right]
= 2a_0 \mE_{\bf z, \sigma} \left[\sup_{w \in B_R} \left\la w, {1\over m} \sum_{i=1}^m \sigma_{i} \Phi(x_i) \right\ra \right].
  \eea
  Using Cauchy-Schwartz inequality, we reach
  \bea
  \mE_{\bf z} \left[\sup_{w \in B_R}(\mcE(w) - \mcE_{\bf z}(w))\right] \leq 2a_0 \mE_{\bf z, \sigma} \left[\sup_{w \in B_R} \|w\| \left\| {1\over m} \sum_{i=1}^m \sigma_{i}\Phi(x_i) \right\| \right]
  \leq 2a_0R \mE_{\bf z, \sigma} \left[ \left\| {1\over m} \sum_{i=1}^m \sigma_{i} \Phi(x_i) \right\| \right].
  \eea
  By Jensen's inequality, we get
  \bea
  \mE_{\bf z} \left[\sup_{w \in B_R}(\mcE(w) - \mcE_{\bf z}(w))\right]
  \leq 2a_0R  \left[ \mE_{\bf z, \sigma} \left\| {1\over m} \sum_{i=1}^m \sigma_{i}\Phi(x_i) \right\|^2 \right]^{1/2}
  = 2a_0R \left[  {1\over m^2} \mE_{\bf z, \sigma} \sum_{i=1}^m \left\|\Phi(x_i) \right\|^2 \right]^{1/2}.
  \eea
  The desired result thus follows by introducing  \eref{boundedkernel} to the above. Note that the above procedure also applies if we replace $\mcE(w) - \mcE_{\bf z}(w)$ with $\mcE_{\bf z}(w) - \mcE(w)$.
  The proof is complete.
\end{proof}

The following lemma gives upper bounds on the iterated sequence.
\begin{lemma}\label{lemma:boundIteration}
 Under Assumption \ref{as:Boundness}. Then for any $t \in \mN$, we have
\bea
   \|w_{t+1}\| \leq \sqrt{ (a_0\kappa)^2 \sum_{k=1}^t \eta_k^2 + 2 |V|_0 \sum_{k=1}^t \eta_{k}}.
\eea
\end{lemma}
\begin{proof}
  Using Lemma \ref{lemma:general} with $w = 0,$ we have
  \bea
  \|w_{k+1}\|^2 \leq \|w_k\|^2 + (a_0\kappa)^2 \eta_k^2 + 2\eta_k \left[ V(y_{j_k}, 0) - V(y_{j_k}, \la w_k, \Phi(x_{j_k}) \ra )\right].
  \eea
  Noting that $V(y,a) \geq 0$ and $V(y_{j_k}, 0) \leq |V|_0,$ we thus get
  \bea
  \|w_{k+1}\|^2 \leq \|w_k\|^2 + (a_0\kappa)^2 \eta_k^2 + 2\eta_k |V|_0.
  \eea
  Applying this inequality iteratively for $k=1,\cdots, t,$ and introducing with $w_1=0$, one can get that
  $$\|w_{t+1}\|^2 \leq  (a_0\kappa)^2 \sum_{k=1}^t \eta_k^2 + 2 |V|_0 \sum_{k=1}^t \eta_{k},$$
  which leads to the desired result by taking square root on both sides.
\end{proof}

According to the above two lemmas, we can bound the sample errors as follows.
\begin{lemma}\label{lemma:sampleError}
 Assume  \eref{boundedkernel} and \eref{boundedDeriviative}. Then, for any $k\in \mN$,
  \bea
  |\mE_{{\bf z},J}[\mcE_{\bf z}(w_k) - \mcE(w_k)]| \leq {2a_0\kappa R_k \over \sqrt{m}},
  \eea
  where
  \be\label{RT1}
  R_k = \sqrt{ (a_0\kappa)^2 \sum_{k=1}^t \eta_k^2 + 2 |V|_0\sum_{k=1}^t \eta_{k}}.
  \ee
\end{lemma}

When the loss function is smooth,
by Theorems 2.2 and 3.9 from \cite{hardt2015train}, we can control the sample errors as follows.
\begin{lemma}\label{lemma:sampleErrorsSmooth}
  Under Assumptions \ref{as:Boundness} and \ref{as:smooth}, let $\eta_t \leq 2/(\kappa^2 L)$  for all $k \in [T],$
  \bea
  \left|\mE_{{\bf z},J}[\mcE_{\bf z}(w_k) - \mcE(w_k) ]\right| \leq  { 2(a_0\kappa)^2\sum_{i=1}^k \eta_i \over m} .
  \eea
\end{lemma}
\begin{proof}
  Note that by \eref{boundedkernel},  Assumption \ref{as:smooth} and \eref{reproducingProperty}, for all $(x,y)\in {\bf z},$ $w,w'\in \mcF,$
  \bea\begin{split} & \|V'(y,\la w, \Phi(x)\ra ) \Phi(x) - V'(y,\la w', \Phi(x) \ra ) \Phi(x) \|
  \leq  \kappa |V'(y, \la w, \Phi(x)\ra ) - V'(y, \la w', \Phi(x)\ra)|  \\
  \leq& \kappa L |\la w, \Phi(x)\ra  - \la w', \Phi(x)\ra| = \kappa L |\la w - w', \Phi(x)\ra | \leq \kappa L \|w - w'\| \|\Phi(x)\|  \\
  \leq& \kappa^2 L \|w-w'\|,\end{split}\eea
  and
  \bea
  \|V'(y,\la w, \Phi(x)\ra ) \Phi(x) \| \leq \kappa a_0.
  \eea
  That is, for every $(x,y)\in {\bf z},$ $V(y, \la \cdot,\Phi(x)\ra)$ is $(\kappa^2 L)$-smooth and $(\kappa a_0)$-Lipschitz.
  Now the results follow directly by using Theorems 2.2 and 3.8 from \cite{hardt2015train}.
\end{proof}
\section{Excess Errors for Weighted Averages}

\begin{lemma}\label{lemma:weightedSum}
  Under Assumption \ref{as:Boundness},
  assume that there exists a non-decreasing sequence $\{b_k>0\}_k$ such that
  \be\label{sampleErrorAssumption}
  \left|\mE_{{\bf z},J}[\mcE_{\bf z}(w_k) - \mcE(w_k) ]\right| \leq b_k, \quad \forall k \in [T] .
  \ee
   Then for any $t\in [T]$ and any fixed $w\in \mcF$,
  \be\label{weightedSumA}
  \sum_{k=1}^t 2\eta_k \mE_{{\bf z},J}\left[ \mcE(w_k) \right] \leq  b_t \sum_{k=1}^t {2 \eta_k} +
     (a_0 \kappa)^2 \sum_{k=1}^t\eta_k^2  + \sum_{k=1}^t 2\eta_k \mcE(w)  + \|w\|^2.
  \ee
\end{lemma}
\begin{proof}
  By Lemma \ref{lemma:empiricalGF}, we have \eref{empiricalGF}.
  Rewriting $ - \mcE_{\bf z}(w_k)$ as
  $$ -\mcE_{\bf z}(w_k) + \mcE(w_k) - \mcE(w_k),$$
  taking the expectation with respect to $J(T)$ and ${\bf z}$ on both sides, noting that $w$ is independent of $J$ and $\bf z$,
  and applying Condition \eref{sampleErrorAssumption},
  we derive
  \bea
  \mE_{{\bf z},J}[\|w_{k+1} - w\|^2] \leq \mE_{{\bf z},J}[\|w_k - w\|^2] + (a_0 \kappa)^2 \eta_k^2  + 2\eta_k \left(\mcE(w) - \mE_{{\bf z},J} \left[ \mcE(w_k)\right] \right) + {2 \eta_k b_k},
  \eea
  which is equivalent to
  \bea
 2\eta_k \mE_{{\bf z},J}\left[ \mcE(w_k) \right] \leq 2\eta_k \mcE(w)  + \mE_{{\bf z},J}[\|w_k - w\|^2 - \|w_{k+1} - w\|^2] + (a_0 \kappa)^2 \eta_k^2   + {2 \eta_k b_k}.
  \eea
  Summing up over $k=1, \cdots,t$, and introducing with $w_1=0,$
  \bea
 \sum_{k=1}^t 2\eta_k \mE_{{\bf z},J}\left[ \mcE(w_k) \right] \leq \sum_{k=1}^t 2\eta_k \mcE(w)  + \|w\|^2 +  (a_0 \kappa)^2 \sum_{k=1}^t\eta_k^2   +  \sum_{k=1}^t {2 \eta_k b_k}.
  \eea
  The proof can be finished by noting that $b_k$ is non-decreasing.
\end{proof}

Now, we are in a position to prove Theorem \ref{thm:convergence}.
\begin{proof}
  [Proof of Theorem \ref{thm:convergence}]
  According to Lemma \ref{lemma:sampleError}, Condition \eref{sampleErrorAssumption} is satisfied for
  $$ b_t = {2a_0\kappa \sqrt{\sum_{k=1}^t (a_0\kappa\eta_k)^2 + 2|V|_0 \sum_{k=1}^t\eta_k}  \over \sqrt{m}}.$$
  By Lemma \ref{lemma:weightedSum}, we thus have \eref{weightedSumA}.
  Dividing both sides by $\sum_{k=1}^{t}2\eta_k,$ and using the convexity of $V(y,\cdot)$ which implies
  \be\label{convexityWei}
  {\sum_{k=1}^t \eta_k \mcE(w_k) \over \sum_{k=1}^t \eta_k }  \geq \mcE({\sum_{k=1}^t \eta_k w_k \over \sum_{k=1}^t \eta_k }) =  \mcE(\overline{w}_t),
  \ee
  we get that
  \bea
   \mE_{{\bf z},J}\left[ \mcE(\overline{w}_t) \right] \leq b_t  +  {(a_0 \kappa)^2 \over 2} {\sum_{k=1}^t\eta_k^2 \over \sum_{k=1}^t \eta_k}  +   \mcE(w)  + {\|w\|^2 \over 2\sum_{k=1}^t \eta_k}.
  \eea
For any fixed $\epsilon>0,$ we know that there exists a $w_{\epsilon}\in \mcF,$ such that $\mcE(w_{\epsilon}) \leq \inf_{w\in\mcF}  \mcE(w) + \epsilon.$
Letting $t = t^*(m),$ and $w = w_{\epsilon},$ we have
\bea
 \mE_{{\bf z},J}\left[ \mcE(w_{t^*(m),w}) \right] \leq  b_{t^*(m)} +  {(a_0 \kappa)^2 \over 2} {\sum_{k=1}^{t^*(m)}\eta_k^2 \over \sum_{k=1}^{t^*(m)} \eta_k}  +  \inf_{w\in\mcF}  \mcE(w) + \epsilon  + {\|w_{\epsilon}\|^2 \over 2\sum_{k=1}^{t^*(m)} \eta_k} .
  \eea
Letting $m\to \infty$ , and using Conditions (A) and (B) which imply
\bea
\lim_{m\to \infty}{1 \over \sum_{k=1}^{t^*(m)} \eta_k}=0,\quad  \lim_{m\to \infty}{\sum_{k=1}^{t^*(m)} \eta_k^2 \over \sum_{k=1}^{t^*(m)} \eta_k}=0,\quad \mbox{and } \lim_{m \to \infty}{\sum_{k=1}^{t^*(m)} \eta_k^2 \over m} =  \lim_{m \to \infty}{\sum_{k=1}^{t^*(m)} \eta_k^2 \over \sum_{k=1}^{t^*(m)} \eta_k} {\sum_{k=1}^{t^*(m)} \eta_k \over m} =0,
\eea
 we reach
\bea
   \lim_{m\to \infty} \mE_{{\bf z},J}\left[ \mcE(w_{t^*(m),w}) \right] \leq  \inf_{w\in\mcF}  \mcE(w) + \epsilon .
\eea
Since $\epsilon>0$ is arbitrary, the desired result thus follows. The proof is complete.
\end{proof}

\begin{lemma}\label{lemma:weightedSumErrors}
  Under the assumptions of Lemma \ref{lemma:weightedSum}, let Assumption \ref{as:approximationerror} hold.
   Then for any $t\in [T],$
  \be\label{averageIntermErrorA}
  \sum_{k=1}^t 2\eta_k \mE_{{\bf z},J}\left[ \mcE(w_k) - \inf_{w\in\mcF}  \mcE(w)\right]
 \leq  b_t { \sum_{k=1}^t 2\eta_k}   + (a_0 \kappa)^2 \sum_{k=1}^t \eta_k^2  + 2c_{\beta}\left(\sum_{k=1}^t \eta_k\right)^{1-\beta}.
  \ee
\end{lemma}
\begin{proof}
  By Lemma \ref{lemma:weightedSum}, we have \eref{weightedSumA}.
  Subtracting $\sum_{k=1}^t 2\eta_k \inf_{w\in\mcF}  \mcE(w)$ from both sides,
  \begin{align*}
  \sum_{k=1}^t 2\eta_k \mE_{{\bf z},J}\left[ \mcE(w_k) - \inf_{w\in\mcF}  \mcE(w)\right]
 \leq   b_t { \sum_{k=1}^t 2\eta_k} + (a_0 \kappa)^2 \sum_{k=1}^t \eta_k^2  + \sum_{k=1}^t 2\eta_k \left[ \mcE(w) - \inf_{w\in\mcF}  \mcE(w)\right] + \|w\|^2.
  \end{align*}
   Taking the infimum over $w\in \mcF$, recalling that $\mathcal{D}(\lambda)$ is defined by \eref{approxerror}, we have
  \bea
  \begin{split}
  \sum_{k=1}^t 2\eta_k \mE_{{\bf z},J}\left[ \mcE(w_k) - \inf_{w\in\mcF}  \mcE(w)\right]
 \leq   b_t { \sum_{k=1}^t 2\eta_k}  + (a_0 \kappa)^2 \sum_{k=1}^t \eta_k^2  + \sum_{k=1}^t 2\eta_k \mathcal{D}\left({1 \over \sum_{k=1}^t\eta_k}\right) .
  \end{split}
  \eea
  Using Assumption \ref{as:approximationerror} to the above, we get the desired result. The proof is complete.
\end{proof}

Collecting some of the above analysis, we get the following result.
\begin{pro}\label{pro:averageIntermError}
  Under the assumptions of Lemma \ref{lemma:weightedSumErrors}, we have
\be\label{averageIntermError}
\mE_{{\bf z},J}[\mcE(\overline{w}_t)] - \inf_{w\in\mcF}  \mcE(w)  \leq b_t + {(a_0 \kappa)^2\over 2} {\sum_{k=1}^t \eta_k^2 \over \sum_{k=1}^{t}\eta_k}  +   c_{\beta} \left({1\over \sum_{k=1}^t \eta_k}\right)^{\beta}.
\ee
\end{pro}
\begin{proof}
  By Lemma \ref{lemma:weightedSumErrors}, we have \eref{averageIntermErrorA}.  Dividing both sides by $\sum_{k=1}^{t}2\eta_k,$ and using \eref{convexityWei}, we get the desired bound.
\end{proof}

\section{From Weighted Averages to the Last Iterate}
A basic tool for studying the convergence for iterates is the following decomposition, as often done in \cite{shamir2013stochastic} for classical online learning or subgradient descent algorithms \cite{lin2015iterative}.
It enables us to study the weighted excess generalization error $2\eta_t\mE_{{\bf z},J}[\mcE(w_t) - \inf_{w\in\mcF}  \mcE(w)]$  in terms of ``weighted averages'' and moving weighted averages.
In what follows, we will write $\mE_{{\bf z},J}$ as $\mE$ for short.
\begin{lemma}We have
  \begin{multline}
 2 \eta_{t} \mE\left\{\mcE(w_t) -\inf_{w\in\mcF}  \mcE(w)\right\} \\
  \leq \frac{1}{t} \sum_{k=1}^{t} 2 \eta_{k} \mE\left\{\mcE(w_{k}) -\inf_{w\in\mcF}  \mcE(w)\right\}
 + \sum_{k=1}^{t-1} \frac{1}{k(k+1)} \sum_{i=t-k+1}^{t} 2 \eta_{i} \mE\left\{\mcE(w_{i}) -\mcE(w_{t-k})\right\}.
\label{errorDecom}
\end{multline} 
\end{lemma}
\begin{proof}
   Let $\{u_t \}_t $ be a real-valued sequence. For $k=1, \cdots, t-1$,
  \bea
  {1 \over k} \sum_{i=t-k+1}^{t} u_{i} - {1 \over k+1} \sum_{i=t-k}^t u_i = {1 \over k(k+1)} \left\{ (k+1)\sum_{i=t-k+1}^{t} u_{i} - k \sum_{i=t-k}^t u_i \right\}
  = {1 \over k(k+1)} \sum_{i=t-k+1}^{t} (u_{i} -  u_{t-k}) .
  \eea
  Summing over $k=1, \cdots, t-1$, and rearranging terms, we get
  \bea
   u_{t}  = {1 \over t} \sum_{i=1}^t u_i + \sum_{k=1}^{t-1} {1 \over k(k+1)} \sum_{i=t-k+1}^{t} (u_{i} -  u_{t-k}) .
  \eea
  Choosing $u_t = 2 \eta_{t} \mE\left\{\mcE(w_t) -\inf_{w\in\mcF}  \mcE(w)\right\}$ in the above, we get
 \begin{align*} 2 \eta_{t} \mE\left\{\mcE(w_t) -\inf_{w\in\mcF}  \mcE(w)\right\}
    = {1 \over t} \sum_{i=1}^t 2 \eta_{i} \mE \left\{ \mcE(w_i) -\inf_{w\in\mcF}  \mcE(w)\right\} \\
    + \sum_{k=1}^{t-1} {1 \over k(k+1)} \sum_{i=t-k+1}^{t} \left(2 \eta_{i} \mE \left\{\mcE(w_i) -\inf_{w\in\mcF}  \mcE(w)\right\}  -  2 \eta_{t-k} \mE \left\{ \mcE(w_{t-k})  -\inf_{w\in\mcF}  \mcE(w)\right\}\right),
    \end{align*}
 which can be rewritten as
\bea
\begin{split}
 &2 \eta_{t} \mE\left\{\mcE(w_t) -\inf_{w\in\mcF}  \mcE(w)\right\}\\
  =& \frac{1}{t} \sum_{k=1}^{t} 2 \eta_{k} \mE\left\{\mcE(w_{k}) -\inf_{w\in\mcF}  \mcE(w)\right\} + \sum_{k=1}^{t-1} \frac{1}{k(k+1)} \sum_{i=t-k+1}^{t} 2 \eta_{i} \mE\left\{\mcE(w_{i}) -\mcE(w_{t-k})\right\}  \\
& + \sum_{k=1}^{t-1} \frac{1}{k+1} \left[\frac{1}{k}\sum_{i=t-k+1}^{t} 2 \eta_{i} - 2 \eta_{t-k}\right] \mE \left\{\mcE(w_{t-k}) - \inf_{w\in\mcF}  \mcE(w)\right\}.
\end{split}\eea
Since, $\mcE(w_{t-k}) - \inf_{w\in\mcF}  \mcE(w) \geq 0$ and that $\{\eta_t\}_{t \in \mN}$ is a non-increasing sequence, we know that the last term of the above inequality is at most zero.
Therefore, we get the desired result. The proof is complete.
\end{proof}
The first term of the right-hand side of \eref{errorDecom} is the weighted excess generalization error, and it can be estimated easily by  \eref{averageIntermErrorA}, while the second term (sum of moving averages) can be estimated by the following lemma.

\begin{lemma}
Under the assumptions of Lemma \ref{lemma:weightedSum}, we have
   \be\label{movingAverage}
   \begin{split}
  & \sum_{k=1}^{t-1} \frac{1}{k(k+1)} \sum_{i=t-k+1}^{t} 2 \eta_{i} \mE\left\{\mcE(w_{i}) -\mcE(w_{t-k})\right\}\\
    \leq& \sum_{i=1}^{t-1} \frac{(a_0 \kappa \eta_i)^2 + 4 b_t\eta_i}{t-i} -\frac{1}{t}\sum_{k=1}^{t}(a_0 \kappa \eta_k)^2 + 4 b_t\eta_k) + (a_0 \kappa \eta_t)^2 + 4 b_t\eta_t.
   \end{split}\ee
\end{lemma}
\begin{proof}
  Given any sample $\bf z,$ note that $w_{t-k}$ is depending only on $j_1,j_2,\cdots,j_{t-k-1}$, and thus is independent from $j_{i+1}$ for any $t\geq i\geq t - k .$ Following from Lemma \ref{lemma:empiricalGF}, for any $i\geq t - k,$
  \bea
  \mE_{j_{i+1}} [\|w_{i+1} - w_{t-k}\|^2] \leq \|w_{i} - w_{t-k}\|^2 + (a_0 \kappa)^2 \eta_i^2  + 2\eta_i \left( \mcE_{\bf z}(w_{t-k}) - \mcE_{\bf z}(w_{i})\right).
  \eea
  Taking the expectation on both sides, and bounding $\mE[\mcE_{\bf z}(w_{t-k}) - \mcE_{\bf z}(w_{i})]$ as
  \bea
  = \mE[ \mcE_{\bf z}(w_{t-k}) - \mcE(w_{t-k}) + \mcE(w_{i}) - \mcE_{\bf z}(w_{i}) + \mcE(w_{t-k}) - \mcE(w_{i})]
  \leq 2 b_t + \mE[\mcE(w_{t-k}) - \mcE(w_{i})]
  \eea
  by Condition \eref{sampleErrorAssumption}, and rearranging terms, we get
  \bea
   2\eta_i \mE\left[\mcE(w_{i}) - \mcE(w_{t-k})\right] \leq \mE[\|w_{i} - w_{t-k}\|^2 - \|w_{i+1} - w_{t-k}\|^2] + (a_0 \kappa)^2 \eta_i^2 + 4\eta_i b_t.
  \eea
  Summing up over $i=t-k,\cdots,t,$ we get
  \bea
    \sum_{i=t-k}^t 2\eta_i \mE\left[\mcE(w_{i}) - \mcE(w_{t-k})\right] \leq (a_0 \kappa)^2 \sum_{i=t-k}^t \eta_i^2 + 4  b_t \sum_{i=t-k}^t \eta_i.
  \eea
  The left-hand side is exactly $\sum_{i=t-k+1}^t 2\eta_i \mE\left[\mcE(w_{i}) - \mcE(w_{t-k})\right].$ Thus, dividing both sides by $k(k+1)$, and then
   summing up  over $k = 1,\cdots,t-1$,
   \bea\begin{split}
   \sum_{k=1}^{t-1} \frac{1}{k(k+1)} \sum_{i=t-k+1}^{t} 2 \eta_{i} \mE\left\{\mcE(w_{i}) -\mcE(w_{t-k})\right\}
    \leq    \sum_{k=1}^{t-1} \frac{1}{k(k+1)} \sum_{i=t-k}^t ((a_0 \kappa \eta_i)^2 + 4 b_t\eta_i).
   \end{split}\eea
   Exchanging the order in the sum, and setting $\xi_i = (a_0 \kappa \eta_i)^2 + 4 b_t\eta_i$ for all $ i \in[t],$ we obtain
\bea
 \sum_{k=1}^{t-1} \frac{1}{k(k+1)} \sum_{i=t-k+1}^{t} 2 \eta_{i} \mE\left\{\mcE(w_{i}) -\mcE(w_{t-k})\right\}
& \leq& \sum_{i=1}^{t-1}\sum_{k=t-i}^{t-1} \frac{1}{k(k+1)} \xi_i +\sum_{k=1}^{t-1}\frac{1}{k(k+1)} \xi_t\\
&=&\sum_{i=1}^{t-1} \left(\frac{1}{t-i}-\frac{1}{t}\right) \xi_i +\left(1-\frac{1}{t}\right)\xi_t\\
&=&\sum_{i=1}^{t-1} \frac{1}{t-i} \xi_i +\xi_t-\frac{1}{t}\sum_{k=1}^{t}\xi_k.
\eea
From the above analysis, we can conclude the proof.
\end{proof}

\begin{pro}\label{pro:IterateErrorInterm}
  Under the assumptions of Lemma \ref{lemma:weightedSumErrors}, we have
  \be\label{IterateErrorInterm}
  \begin{split}
 \mE_{{\bf z},J} [\mcE(w_t) - \inf_{w\in\mcF}  \mcE(w)] \leq  b_t\left(1 +  \sum_{k=1}^{t-1} \frac{ 2\eta_k}{\eta_t(t-k)} \right) + \sum_{k=1}^{t-1} \frac{(a_0 \kappa \eta_k)^2}{2\eta_t(t-k)} + {(a_0 \kappa )^2 \eta_t\over 2}  + {c_{\beta}\over \eta_t t}\left(\sum_{k=1}^t \eta_k\right)^{1-\beta}
  \end{split}\ee
\end{pro}
\begin{proof}
  Plugging \eref{averageIntermErrorA} and \eref{movingAverage} into \eref{errorDecom},
by a direct calculation, we get
\bea
\begin{split}
 2\eta_t \mE_{{\bf z},J} [\mcE(w_t) - \inf_{w\in\mcF}  \mcE(w)]
\leq  {2c_{\beta}\over t}\left(\sum_{k=1}^t \eta_k\right)^{1-\beta}  + \sum_{k=1}^{t-1} \frac{(a_0 \kappa \eta_k)^2 + 4 b_t\eta_k}{t-k} -\frac{2b_t}{t}\sum_{k=1}^{t}\eta_k + (a_0 \kappa \eta_t)^2 + 4 b_t\eta_t.
\end{split}
\eea
Since $\{\eta_t\}_t$ is non-increasing, $\frac{2b_t}{t}\sum_{k=1}^{t}\eta_k \geq 2b_t\eta_t$. Thus,
\bea
 2\eta_t \mE_{{\bf z},J} [\mcE(w_t) - \inf_{w\in\mcF}  \mcE(w)] \leq {2c_{\beta}\over t}\left(\sum_{k=1}^t \eta_k\right)^{1-\beta} + 2\eta_t b_t + \sum_{k=1}^{t-1} \frac{(a_0 \kappa \eta_k)^2 + 4 b_t\eta_k}{t-k} + (a_0 \kappa \eta_t)^2.
\eea
Dividing both sides with $2\eta_t$, and rearranging terms, one can conclude the proof.
\end{proof}

Now, we are ready to prove Theorems \ref{thm:errorSmooth} and \ref{thm:errorGeneral}.

\begin{proof}
  [Proof of Theorem \ref{thm:errorSmooth}]
   By Lemma \ref{lemma:sampleErrorsSmooth}, the condition (\ref{sampleErrorAssumption}) is satisfied with $b_k = {2(a_0\kappa)^2\sum_{i=1}^k \eta_i / m }.$
  It thus follows from Propositions \ref{pro:averageIntermError} and \ref{pro:IterateErrorInterm} that
  \be\label{weightedAverBoundSmooth}
\mE_{{\bf z},J}[\mcE(\overline{w}_t)] - \inf_{w\in\mcF}  \mcE(w)  \leq 2(a_0\kappa)^2 {\sum_{k=1}^t \eta_k \over m }   + {(a_0 \kappa)^2\over 2} {\sum_{k=1}^t \eta_k^2 \over \sum_{k=1}^{t}\eta_k}  + c_{\beta} \left({1\over \sum_{k=1}^t \eta_k}\right)^{\beta} ,
\ee
and
  \be\label{smoothIterateBound}
  \begin{split}
  &\mE_{{\bf z},J} [\mcE(w_t) - \inf_{w\in\mcF}  \mcE(w)] \\
   \leq&  2(a_0\kappa)^2 {\sum_{k=1}^t \eta_k \over m}\left( 1 + \sum_{k=1}^{t-1} \frac{ 2 \eta_k}{\eta_t(t-k)} \right)
   + {(a_0 \kappa)^2 \over 2} \sum_{k=1}^{t-1} \frac{ \eta_k^2}{\eta_t(t-k)} + {(a_0 \kappa)^2  \over 2} \eta_t + c_{\beta} {\left(\sum_{k=1}^t \eta_k\right)^{1-\beta} \over \eta_t t} .
  \end{split}\ee
  By noting that $1 \leq \eta_{t-1}/\eta_t \leq \sum_{k=1}^{t-1}\eta_k/(\eta_t(t-k)),$
  \be\label{smoothIterateBoundA}
  \begin{split}
  &\mE_{{\bf z},J} [\mcE(w_t) - \inf_{w\in\mcF}  \mcE(w)] \\
  \leq&   6(a_0\kappa)^2 {\sum_{k=1}^t \eta_t \over m} \sum_{k=1}^{t-1} \frac{ \eta_k}{\eta_t(t-k)}
   + {(a_0 \kappa)^2 \over 2} \sum_{k=1}^{t-1} \frac{ \eta_k^2}{\eta_t(t-k)} + {(a_0 \kappa)^2  \over 2} \eta_t + c_{\beta} {\left(\sum_{k=1}^t \eta_k\right)^{1-\beta} \over \eta_t t} .
  \end{split}\ee
  The proof is complete.
\end{proof}

\begin{proof}
  [Proof of Theorems \ref{thm:errorGeneral}]
  By Propositions \ref{pro:averageIntermError} and \ref{pro:IterateErrorInterm}, we have \eref{averageIntermError} and \eref{IterateErrorInterm}.
  Also, by Lemma \ref{lemma:sampleError}, we have $b_t \leq {2a_0\kappa R_t \over \sqrt{m}}.$ Then
  \be\label{WeightedAverDetailedGeneral}
\mE_{{\bf z},J}[\mcE(\overline{w}_t)] - \inf_{w\in\mcF}  \mcE(w)  \leq 2a_0 \kappa { R_t \over \sqrt{m}}
  + {(a_0 \kappa)^2\over 2} {\sum_{k=1}^t \eta_k^2 \over \sum_{k=1}^{t}\eta_k} + c_{\beta} \left({1\over \sum_{k=1}^t \eta_k}\right)^{\beta},
\ee
and
  \be\label{IteratesGeneralBound}
  \begin{split}
  &\mE_{{\bf z},J} [\mcE(w_t) - \inf_{w\in\mcF}  \mcE(w)] \\
  \leq&  2a_0\kappa { R_t \over \sqrt{m}}\left( 1 + \sum_{k=1}^{t-1} \frac{ 2 \eta_k}{\eta_t(t-k)} \right)  + {(a_0 \kappa)^2 \over 2}\sum_{k=1}^{t-1} \frac{ \eta_k^2}{\eta_t(t-k)} + {(a_0 \kappa)^2  \over 2} \eta_t + c_{\beta} {\left(\sum_{k=1}^t \eta_k\right)^{1-\beta} \over \eta_t t}.
  \end{split}\ee
  Note that $1 \leq \eta_{t-1}/\eta_t$ since $\eta_t$ is non-increasing. Thus,
  \be\label{IteratesGeneralBoundA}
  \begin{split}
  \mE_{{\bf z},J} [\mcE(w_t) - \inf_{w\in\mcF}  \mcE(w)] \leq  6a_0\kappa { R_t \over \sqrt{m}} \sum_{k=1}^{t-1} \frac{ \eta_k}{\eta_t(t-k)}
   + {(a_0 \kappa)^2 \over 2}\sum_{k=1}^{t-1} \frac{ \eta_k^2}{\eta_t(t-k)} + {(a_0 \kappa)^2  \over 2} \eta_t + c_{\beta} {\left(\sum_{k=1}^t \eta_k\right)^{1-\beta} \over \eta_t t}.
  \end{split}\ee
  Recall that $R_t$ is given by \eref{RT1} and that  $\eta_k$ is non-increasing, we thus have
  \be\label{RT} R_t
  \leq  \sqrt{(a_0\kappa)^2 \eta_1 + 2|V|_0} \sqrt{ \sum_{k=1}^t\eta_k} .\ee
  Introducing the above into \eref{WeightedAverDetailedGeneral} and \eref{IteratesGeneralBoundA},
  \be\label{WeightedAverDetailedSpecialCases}
  \begin{split}
 \mE_{{\bf z},J}[\mcE(\overline{w}_t)] - \inf_{w\in\mcF}  \mcE(w)
\leq 2a_0 \kappa \sqrt{(a_0\kappa)^2 \eta_1 + 2|V|_0} \sqrt{   \sum_{k=1}^t\eta_k \over m}  + {(a_0 \kappa)^2\over 2} {\sum_{k=1}^t \eta_k^2 \over \sum_{k=1}^{t}\eta_k}  +  c_{\beta} \left({1\over \sum_{k=1}^t \eta_k}\right)^{\beta},
\end{split}\ee
and
\bea
  \begin{split}
   &\mE_{{\bf z},J} [\mcE(w_t) - \inf_{w\in\mcF}  \mcE(w)]\\ 
   \leq& 6a_0 \kappa \sqrt{(a_0\kappa)^2 \eta_1 + 2|V|_0} \sqrt{   \sum_{k=1}^t\eta_k \over m}
   + {(a_0 \kappa)^2 \over 2}\sum_{k=1}^{t-1} \frac{ \eta_k^2}{\eta_t(t-k)} + {(a_0 \kappa)^2  \over 2} \eta_t + c_{\beta} {\left(\sum_{k=1}^t \eta_k\right)^{1-\beta} \over \eta_t t}.
  \end{split}\eea
\end{proof}

\section{Explicit Convergence Rates}\label{appendix:basicEstimates}
In this section, we prove Corollaries \ref{cor:smoothExplicitC}-\ref{cor:genExplicitSingle}. We first introduce the following basic estimates.
\begin{lemma}\label{lemma:basicEstimate}
  Let $\theta\in \mR_+$, and $t\in\mN$, $t\geq 3$. Then
 $$ \sum_{k=1}^{t} k^{-\theta} \leq
 \left\{ \begin{array}
   {ll}
   {t^{1 - \theta}/(1-\theta)}, & \hbox{when} \ \theta<1,\\
   {\log t + 1},                          & \hbox{when} \ \theta = 1,\\
   {\theta/(\theta - 1)},       & \hbox{when} \ \theta >1, \\
 \end{array}\right.
 $$
 and
 \bea
  \sum_{k=1}^t k^{-\theta} \geq \left\{ \begin{array}
    {ll}
     {1 - 4^{\theta - 1} \over 1-\theta} t^{1-\theta}  & \hbox{when} \ \theta <1,\\
     \ln t                         & \hbox{when} \ \theta = 1.
  \end{array}
  \right.
  \eea
\end{lemma}
\begin{proof}
By using
  $$ \sum_{k=1}^{t} k^{-\theta} = 1+ \sum_{k=2}^{t} \int_{k-1}^k d u k^{-\theta}  \leq 1 + \sum_{k=2}^t \int_{k-1}^k u^{-\theta} d u = 1  + \int_{1}^t u^{-\theta} d u,
 $$
 which leads to the first part of the result.
 Similarly,
  \bea
  \sum_{k=1}^t k^{-\theta} = \sum_{k=1}^t k^{-\theta} \geq \sum_{k=1}^t \int_{k}^{k+1}u^{-\theta} d u = \int_{1}^{t+1} u^{-\theta} d u,
  \eea
  which leads to the second part of the result. The proof is complete.
\end{proof}

\begin{lemma}\label{Lemma:EstimatingTerm2}
Let  $q \in \mR_+$ and $t\in\mN$, $t\geq 3$. Then
  \bea  \sum_{k=1}^{t-1} {1 \over t-k} k^{-q}
 \leq
 \left\{ \begin{array}
   {ll}
   2^{q}[2+  (1-q)^{-1}] t^{- q} \log t, & \hbox{when} \ q < 1,\\
   {8 t^{-1} \log t},                          & \hbox{when} \ q = 1,\\
   (2^{q}+2q)/(q-1)t^{-1},       & \hbox{when} \ q >1, \\
 \end{array}\right.
  \eea
\end{lemma}
\begin{proof}
We  split the sum into two parts
 \bea
 \sum_{k=1}^{t-1} {1 \over t-k} k^{-q} &=& \sum_{t/2 \leq k \leq t -1 }  {1 \over t-k} k^{-q} + \sum_{1\leq k < t/2}  {1 \over t-k} k^{-q} \\
 & \leq & 2^q t^{-q} \sum_{t/2 \leq k \leq t-1}  {1 \over t-k} + 2 t^{-1} \sum_{1\leq k < t/2}   k^{-q} \\
 & = & 2^q t^{-q} \sum_{1 \leq k \leq t/2}   k^{-1} + 2 t^{-1} \sum_{1\leq k < t/2}   k^{-q}.
\eea
Applying Lemma \ref{lemma:basicEstimate},
 we get
 $$ \sum_{k=1}^{t-1} {1 \over t-k} k^{-q}  \leq 2^{q} t^{-q} (\log (t/2)+1) +
 \left\{ \begin{array}
   {ll}
   2^qt^{ - q}/(1-q), & \hbox{when} \ q < 1,\\
   {4 t^{-1} \log t},                          & \hbox{when} \ q = 1,\\
   {2q} t^{-1}/(q-1),       & \hbox{when} \ q >1, \\
 \end{array}\right.
 $$
which leads to the desired result by using  $t^{-q+1} \log t \leq 1/(\mathrm{e}(q-1)) \leq 1/(\mathrm{2}(q-1))$ when $q> 1$.
\end{proof}

The bounds in the above two lemmas involve constant factor $1/(1-\theta)$ or $1/(1-q)$, which tend to be infinity  as $\theta \to 1$ or $q\to 1.$
To avoid these, we introduce the following complement results.
\begin{lemma}
  \label{lemma:basicEstimateA}
  Let $\theta \in \mR_+,$ and $t \in \mN, $ with $t\geq 3.$
  Then $$ \sum_{k=1}^{t} k^{-\theta} \leq t^{\max(1-\theta,0)} 2 \log t.
 $$
\end{lemma}
\begin{proof}
  Note that
  \bea
  \sum_{k=1}^{t} k^{-\theta} = \sum_{k=1}^{t} k^{-1} k^{1-\theta} \leq t^{\max(1-\theta,0)} \sum_{k=1}^{t} k^{-1}.
  \eea
  The proof can be finished by applying Lemma \ref{lemma:basicEstimate}.
\end{proof}

\begin{lemma}\label{Lemma:EstimatingTerm2A}
Let  $q \in \mR_+$ and $t\in\mN$, $t\geq 3$. Then
  \bea  \sum_{k=1}^{t-1} {1 \over t-k} k^{-q}
 \leq 4 t^{-\min(q,1)} \log t.
  \eea
\end{lemma}
\begin{proof}
  Note that
  \bea
  \sum_{k=1}^{t-1} {1 \over t-k} k^{-q} = \sum_{k=1}^{t-1} {k^{1-q} \over (t-k)k} \leq t^{\max(1-q,0)} \sum_{k=1}^{t-1} {1 \over (t-k)k},
  \eea
 and that by Lemma \ref{lemma:basicEstimate},
   \bea
  \sum_{k=1}^{t-1} {1 \over (t-k)k} = {1 \over t}\sum_{k=1}^{t-1} \left({1 \over t-k} + {1 \over k}\right) = {2 \over t}\sum_{k=1}^{t-1} {1 \over k} \leq {4\over t} \log t.
  \eea
\end{proof}

With the above estimates and Theorems \ref{thm:errorSmooth}, \ref{thm:errorGeneral}, we can get the following two propositions.

\begin{pro}\label{pro:totalSmooth}
  Under Assumptions \ref{as:Boundness}, \ref{as:approximationerror} and \ref{as:smooth}, let $\eta_t =\eta t^{-\theta}$ for some positive constant $\eta \leq {1 \over \kappa^2 L} $ with $\theta \in[0,1)$ for all $t \in \mN.$ Then for all $t \in \mN,$
\bea
\mE_{{\bf z},J}[\mcE(\overline{w}_t)] - \inf_{w\in\mcF}  \mcE(w)  \leq {2(a_0\kappa)^2 \over 1-\theta} {\eta t^{1-\theta} \over m }   + {(a_0 \kappa)^2 (1 - \theta) \over 1 - 4^{\theta-1}} \eta t^{-\min(\theta,1-\theta)} \log t    + c_{\beta}  \left({1-\theta \over 1 - 4^{\theta-1} }\right)^{\beta}  \left({1\over \eta t^{1-\theta}}\right)^{\beta},
\eea
and
\bea
  \mE_{{\bf z},J} [\mcE(w_t) - \inf_{w\in\mcF}  \mcE(w)] \leq  {18(a_0\kappa)^2 \over 1-\theta }  { \eta t^{1-\theta} \log t  \over m}
   + {3(a_0 \kappa)^2}\eta t^{-\min(\theta,1-\theta)} \log t  + {c_{\beta} \over 1-\theta} \left( {1 \over \eta t^{1-\theta}} \right)^{\beta}.
\eea
\end{pro}

\begin{proof}
  Following the proof of Theorem \ref{thm:errorSmooth}, we have \eref{weightedAverBoundSmooth} and \eref{smoothIterateBound}. \\
  We first consider the case $ \overline{w}_t.$
  With $\eta_t = \eta t^{-\theta},$ \eref{weightedAverBoundSmooth} reads as
  \bea
\mE_{{\bf z},J}[\mcE(\overline{w}_t)] - \inf_{w\in\mcF}  \mcE(w)  \leq 2(a_0\kappa)^2 {\eta \sum_{k=1}^t k^{-\theta} \over m }   + {(a_0 \kappa)^2\over 2} {\eta \sum_{k=1}^t k^{-2\theta} \over \sum_{k=1}^{t} k^{-\theta}}  + c_{\beta} \left({1\over \eta \sum_{k=1}^t k^{-\theta}}\right)^{\beta}.
\eea
Lemma \ref{lemma:basicEstimate} tells us that
\bea
{1 - 4^{\theta-1} \over 1-\theta} t^{1-\theta} \leq \sum_{k=1}^t k^{-\theta}  \leq {t^{1-\theta} \over 1-\theta }.
\eea
Thus,
\bea
\mE_{{\bf z},J}[\mcE(\overline{w}_t)] - \inf_{w\in\mcF}  \mcE(w)  \leq {2(a_0\kappa)^2 \over 1-\theta} {\eta t^{1-\theta} \over m }   + {(a_0 \kappa)^2 (1 - \theta) \over 2(1 - 4^{\theta-1})} {\eta \sum_{k=1}^t k^{-2\theta} \over t^{1-\theta}}   + c_{\beta}  \left({1-\theta \over 1 - 4^{\theta-1} }\right)^{\beta}  \left({1\over \eta t^{1-\theta}}\right)^{\beta} .
\eea
Using Lemma \ref{lemma:basicEstimateA} to the above, we can get the first part of the desired results.\\
Now consider the case $w_t.$ With $\eta_t = \eta t^{-\theta},$ \eref{smoothIterateBound} is exactly
  \bea
  \begin{split}
  \mE_{{\bf z},J} [\mcE(w_t) - \inf_{w\in\mcF}  \mcE(w)] \leq  2(a_0\kappa)^2 { \eta\sum_{k=1}^t k^{-\theta} \over m}\left( 1 + 2t^{\theta} \sum_{k=1}^{t-1} \frac{  k^{-\theta} }{t-k} \right)
   \\+ {(a_0 \kappa)^2 \over 2}\eta t^{-\theta} \left( t^{2\theta} \sum_{k=1}^{t-1} \frac{ k^{-2\theta} }{t-k} + 1\right) + c_{\beta} {\left( \sum_{k=1}^t k^{-\theta}\right)^{1-\beta} \over \eta^{\beta} t^{1-\theta}} .
  \end{split}\eea
Applying Lemma \ref{Lemma:EstimatingTerm2A} to bound $\sum_{k=1}^{t-1} \frac{  k^{-\theta} }{t-k}$ and $\sum_{k=1}^{t-1} \frac{ k^{-2\theta} }{t-k},$ by a simple calculation, we derive
\bea
  \begin{split}
  \mE_{{\bf z},J} [\mcE(w_t) - \inf_{w\in\mcF}  \mcE(w)] \leq  2(a_0\kappa)^2 { \eta  \sum_{k=1}^t k^{-\theta} \over m} \cdot (9\log t)
   + {3(a_0 \kappa)^2}\eta t^{-\min(\theta,1-\theta)} \log t  + c_{\beta} {\left( \sum_{k=1}^t k^{-\theta}\right)^{1-\beta} \over \eta^{\beta} t^{1-\theta}} .
  \end{split}\eea
Using Lemma \ref{lemma:basicEstimate} to upper bound $\sum_{k=1}^t k^{-\theta}$, one can get the second part of the desired results.
\end{proof}

\begin{pro}\label{pro:totalGeneral}
  Under Assumptions \ref{as:Boundness} and \ref{as:approximationerror}, let $\eta_t = \eta t^{-\theta}$ for all $t \in \mN$, with $0<\eta\leq 1$ and $\theta \in[0,1)$.
  Then for all $t\in \mN,$
  \bea
  \begin{split}
\mE_{{\bf z},J}[\mcE(\overline{w}_t)] - \inf_{w\in\mcF}  \mcE(w)  \leq  2a_0\kappa \sqrt{(a_0 \kappa)^2 + 2|V|_0 \over 1-\theta} \sqrt{\eta t^{1-\theta} \over m}    + {(a_0 \kappa)^2 (1 - \theta) \over 1 - 4^{\theta-1}} \eta t^{-\min(\theta,1-\theta)} \log t   \\
 + c_{\beta}  \left({1-\theta \over 1 - 4^{\theta-1} }\right)^{\beta}  \left({1\over \eta t^{1-\theta}}\right)^{\beta},
\end{split}\eea
and
\bea
  &&\mE_{{\bf z},J} [\mcE(w_t) - \inf_{w\in\mcF}  \mcE(w)] \\
  &&\leq 18 a_0\kappa \sqrt{(a_0 \kappa)^2 + 2|V|_0 \over 1-\theta} \sqrt{\eta t^{1-\theta} \over m} \log t
   + {3(a_0 \kappa)^2}\eta t^{-\min(\theta,1-\theta)} \log t  + {c_{\beta} \over 1-\theta} \left( {1 \over \eta t^{1-\theta}} \right)^{\beta}.
\eea
\end{pro}

\begin{proof}
  Following the proof of Theorem \ref{thm:errorGeneral}, we have \eref{WeightedAverDetailedGeneral} and \eref{IteratesGeneralBound}, where $R_t$ satisfies \eref{RT}.
 Comparing  \eref{WeightedAverDetailedGeneral}, \eref{IteratesGeneralBound} with \eref{weightedAverBoundSmooth}, \eref{smoothIterateBound}, we find that the differences are the terms related sample errors, i.e., the term $2(a_0\kappa^2) \sum_{k=1}^t \eta_k /m$ in \eref{weightedAverBoundSmooth}, \eref{smoothIterateBound},
 while $2a_0\kappa R_t /\sqrt{m}$ in \eref{WeightedAverDetailedGeneral}, \eref{IteratesGeneralBound}. Thus, following from the proof of Proposition \ref{pro:totalSmooth}, we get
 \bea
\mE_{{\bf z},J}[\mcE(\overline{w}_t)] - \inf_{w\in\mcF}  \mcE(w)  \leq {2a_0\kappa R_t \over \sqrt{m}}   + {(a_0 \kappa)^2 (1 - \theta) \over 1 - 4^{\theta-1}} \eta t^{-\min(\theta,1-\theta)} \log t    + c_{\beta}  \left({1-\theta \over 1 - 4^{\theta-1} }\right)^{\beta}  \left({1\over \eta t^{1-\theta}}\right)^{\beta},
\eea
and
\bea
  \mE_{{\bf z},J} [\mcE(w_t) - \inf_{w\in\mcF}  \mcE(w)] \leq { 2a_0\kappa R_t \over \sqrt{m}} \cdot 9\log t
   + {3(a_0 \kappa)^2}\eta t^{-\min(\theta,1-\theta)} \log t  + {c_{\beta} \over 1-\theta} \left( {1 \over \eta t^{1-\theta}} \right)^{\beta}.
\eea
Recall that $R_t$ satisfies \eref{RT}, with $\eta_t = \eta t^{-\theta}$, where $ \eta \leq 1$, by Lemma \ref{lemma:basicEstimate}, we know that
\bea
R_t \leq \sqrt{(a_0 \kappa)^2 + 2|V|_0 \over 1-\theta} \sqrt{\eta t^{1-\theta}}.
\eea
From the above analysis, one can conclude the proof.
\end{proof}

We are ready to prove Corollaries \ref{cor:smoothExplicitC}-\ref{cor:genExplicitSingle}.

\begin{proof}[Proof of Corollary \ref{cor:smoothExplicitC}]
  Applying Proposition \ref{pro:totalSmooth} with $\theta = 0$, $\eta = \eta_1 /\sqrt{m}$, we derive
  \bea
\mE_{{\bf z},J}[\mcE(\overline{w}_t)] - \inf_{w\in\mcF}  \mcE(w)  \leq   2\eta_1 (a_0\kappa)^2 {t \over \sqrt{m^3} } + 2{(a_0 \kappa)^2\eta_1}{\log t \over \sqrt{m}} + {2c_{\beta} \over \eta_1^{\beta}} \left({\sqrt{m} \over t }\right)^{\beta},
\eea
and
\bea
  \mE_{{\bf z},J} [\mcE(w_t) - \inf_{w\in\mcF}  \mcE(w)] \leq  18\eta_1 (a_0\kappa)^2 {t\log t \over \sqrt{m^3}} + {3\eta_1 (a_0 \kappa)^2} {\log t \over \sqrt{m}} + {c_{\beta} \over \eta_1^{\beta}} \left({\sqrt{m} \over t }\right)^{\beta}.
\eea
The proof is complete.
\end{proof}

\begin{proof}
  [Proof of Corollary  \ref{cor:smoothExplicit}]
   Applying Proposition \ref{pro:totalSmooth} with $\eta=\eta_1,\theta =1/2,$ we get
  \bea
  \begin{split}
 \mE_{{\bf z},J}[\mcE(\overline{w}_t)] - \inf_{w\in\mcF}  \mcE(w)
\leq 4(a_0\kappa)^2 \eta_1 {  \sqrt{t}  \over m}  + (a_0 \kappa)^2 \eta_1 {\log t \over \sqrt{t}}  + c_{\beta}\eta_1^{-\beta} {1 \over t^{\beta/2}} ,
\end{split}\eea
and
\bea
  \mE_{{\bf z},J} [\mcE(w_t) - \inf_{w\in\mcF}  \mcE(w)]
  \leq   36(a_0\kappa)^2\eta_1 {\sqrt{t}\log t \over m}
 + 3(a_0 \kappa)^2 \eta_1 {\log t \over \sqrt{t}} + 2c_{\beta}\eta_1^{-\beta} {1 \over t^{\beta/2}}.
 \eea
\end{proof}

\begin{proof}[Proof of Corollary \ref{cor:smoothExplicitCSingle}]
  Applying Proposition \ref{pro:totalSmooth} with $\theta =0$ and $\eta = \eta_1 m^{-{\beta \over \beta+1}},$ we get
\bea
\mE_{{\bf z},J}[\mcE(\overline{w}_t)] - \inf_{w\in\mcF}  \mcE(w)  \leq 2 \eta_1 (a_0\kappa)^2 {  m^{-{\beta +2 \over \beta+1}} t }+  { 2\eta_1 (a_0 \kappa)^2} m^{-{\beta \over \beta+1}} \log t  + {2c_{\beta} \over \eta_1^{\beta}} m^{{\beta^2 \over \beta+1}} t^{-\beta},
\eea
and
 \bea
  \mE_{{\bf z},J} [\mcE(w_t) - \inf_{w\in\mcF}  \mcE(w)] \leq  18\eta_1 (a_0\kappa)^2 { m^{-{\beta +2 \over \beta+1}} t \log t}  + {  3\eta_1 (a_0 \kappa)^2 } m^{-{\beta \over \beta+1}} \log t + {c_{\beta} \over \eta_1^{\beta}} m^{{\beta^2 \over \beta+1}} t^{-\beta}.
\eea
The proof is complete.
\end{proof}

\begin{proof}[Proof of Corollary \ref{cor:smoothExplicitDSingle}]
 Applying Proposition \ref{pro:totalSmooth} with $\eta = \eta_1$ and $\theta = {\beta \over \beta+1},$
\bea
\mE_{{\bf z}, J}[\mcE(\overline{w}_t)] - \inf_{w\in\mcF}  \mcE(w) \leq 4(a_0\kappa)^2 \eta_1 {t^{1\over \beta+1} \over m} + {2\eta_1 (a_0\kappa)^2} t^{-{\beta  \over \beta+1}}  + 2c_{\beta} \eta_1^{-\beta} t^{-{\beta \over \beta+1}},
\eea
and
\bea\begin{split}
  \mE_{{\bf z},J} [\mcE(w_t) - \inf_{w\in\mcF}  \mcE(w)] \leq  36(a_0\kappa)^2\eta_1 {t^{1\over 1+ \beta}\log t \over m} + {3\eta_1 (a_0 \kappa)^2} t^{-{\beta \over \beta+1} }\log t + 2c_{\beta} \eta_1^{-\beta} t^{-{\beta \over \beta+1}}.
  \end{split}\eea
For the above two inequalities, we used that $\beta \in (0,1],$ $\theta = {\beta\over \beta+1} \leq 1/2$ and $4^{\theta -1} \leq 1/2.$
\end{proof}

\begin{proof}
  [Proof of Corollary \ref{cor:genExplicitC}]
Applying Proposition \ref{pro:totalGeneral} with $\eta = 1/\sqrt{m}$ and $\theta = 0$,
\bea
 \mE_{{\bf z},J}[\mcE(\overline{w}_t)] - \inf_{w\in\mcF}  \mcE(w)
\leq 2a_0 \kappa \sqrt{(a_0\kappa)^2 + 2|V|_0} {\sqrt{ t} \over m^{3/4}}  + {2(a_0 \kappa)^2} { \log t \over \sqrt{m}}  +  2c_{\beta} \left({\sqrt{m} \over t}\right)^{\beta},
\eea
and
\bea
   \mE_{{\bf z},J} [\mcE(w_t) - \inf_{w\in\mcF}  \mcE(w)] \leq
 18a_0 \kappa \sqrt{(a_0\kappa)^2 + 2|V|_0} {\sqrt{ t}\log t \over m^{3/4}}  + {3(a_0 \kappa)^2}{ \log t \over \sqrt{m}} + c_{\beta} \left({\sqrt{m} \over t}\right)^{\beta} .
\eea
The proof is complete.
\end{proof}

\begin{proof}
  [Proof of Corollary \ref{cor:genExplicit}]
  Applying Proposition \ref{pro:totalGeneral} with $\eta=1$ and $\theta =1/2,$ we get
  \bea
  \begin{split}
     \mE_{{\bf z},J}[\mcE(\overline{w}_t)] - \inf_{w\in\mcF}  \mcE(w)
\leq  2\sqrt{2} a_0 \kappa \sqrt{(a_0\kappa)^2 + 2|V|_0} {t^{1/4} \over \sqrt{m}} + {(a_0 \kappa)^2\log t \over \sqrt{t}}   + {c_{\beta} \over t^{\beta/2}}.
\end{split}\eea
and
  \bea
  \begin{split}
  \mE_{{\bf z},J} [\mcE(w_t) - \inf_{w\in\mcF}  \mcE(w)]
  \leq  18\sqrt{2} a_0 \kappa \sqrt{(a_0\kappa)^2 + 2|V|_0} {t^{1/4}\log t \over \sqrt{m}} + 3(a_0 \kappa)^2{ \log t\over \sqrt{t}}  +  2c_{\beta} { 1 \over t^{\beta / 2}}.
  \end{split}
  \eea
\end{proof}

\begin{proof}
  [Proof of Corollary \ref{cor:genExplicitCSingle}]
  Using Proposition \ref{pro:totalGeneral} with $\eta = m^{-{2\beta \over 2\beta+1}}$ and $\theta =0,$ we get
 \bea
 \mE_{{\bf z},J}[\mcE(\overline{w}_t)] - \inf_{w\in\mcF}  \mcE(w)
\leq 2a_0 \kappa \sqrt{(a_0\kappa)^2 + 2|V|_0} m^{-{4\beta+1 \over 4\beta+2}}\sqrt{ t}  + {2(a_0 \kappa)^2} m^{-{2\beta \over 2\beta+1}} \log t  +  2c_{\beta} m^{{2\beta^2 \over 2\beta+1}}  t^{-\beta} ,
\eea
and
\bea
   \mE_{{\bf z},J} [\mcE(w_t) - \inf_{w\in\mcF}  \mcE(w)] \leq
   18a_0 \kappa \sqrt{(a_0\kappa)^2 \eta + 2|V|_0} m^{-{4\beta+1 \over 4\beta+2}}\sqrt{ t} \log t
   + {3(a_0 \kappa)^2}m^{-{2\beta \over 2\beta+1}} \log t  +c_{\beta}  m^{{2\beta^2 \over 2\beta+1}}  t^{-\beta}.
\eea
The proof is complete.
\end{proof}

\begin{proof}
  [Proof of Corollary \ref{cor:genExplicitSingle}]
Let $\theta = {2\beta \over 2\beta+1}.$ Obviously, $\theta \in [0,{2\over 3}]$ since $\beta \in(0,1]$.
Thus, ${1 \over 1-\theta} = 2\beta + 1 \leq 3,$ ${1-\theta \over 1 - 4^{\theta -  1}} \leq {1 \over 1- 4^{-1/3}} \leq 2. $
Following from Proposition \ref{pro:totalGeneral},
\bea
  \begin{split}
\mE_{{\bf z},J}[\mcE(\overline{w}_t)] - \inf_{w\in\mcF}  \mcE(w)  \leq  2\sqrt{3} a_0\kappa \sqrt{(a_0 \kappa)^2 + 2|V|_0}  {t^{1 \over 4\beta+2} \over \sqrt{m} }    + 2(a_0 \kappa)^2 t^{-{\min(2\beta,1)\over 2\beta+1}} \log t  + 2 c_{\beta}   t^{-{\beta \over 2\beta+1}},
\end{split}\eea
and
\bea
  \mE_{{\bf z},J} [\mcE(w_t) - \inf_{w\in\mcF}  \mcE(w)] \leq 18\sqrt{3} a_0\kappa \sqrt{(a_0 \kappa)^2 + 2|V|_0 } {t^{1 \over 4\beta+2} \over \sqrt{m} } \log t + {3(a_0 \kappa)^2}  t^{-{\min(2\beta,1)\over 2\beta+1}} \log t  + {3 c_{\beta}} t^{-{\beta \over 2\beta+1}}.
\eea
\end{proof}

\end{document}